\documentclass[]{article}

\usepackage{mkolar_definitions}

\RequirePackage{natbib}

\usepackage[boxed,noend]{algorithm2e}
\usepackage{graphicx}
\usepackage{subfigure}

\bibpunct{(}{)}{;}{a}{,}{,}

\theoremstyle{plain}

\def\bks{\backslash}
\newcommand{\iba}{_{i, \bks a}}
\newcommand{\ba}{_{\bks a}}
\def\0b{\mathbf{0}}
\newcommand{\dotpm}[2]{\langle\!\langle{#1},{#2}\rangle\!\rangle}

\newcommand{\supb}[1]{^{\hat \Bcal^{#1}}}

\newcommand{\sk}{_{ S^k}}

\title{Estimating Networks With Jumps}
\author{Mladen Kolar and Eric P. Xing\\
School of Computer Science\\
Carnegie Mellon University}
\date{}

\begin{document}

\maketitle

\begin{abstract}
  We study the problem of estimating a temporally varying coefficient
  and varying structure (VCVS) graphical model underlying
  nonstationary time series data, such as social states of interacting
  individuals or microarray expression profiles of gene networks, as
  opposed to {\it i.i.d.} data from an invariant model widely
  considered in current literature of structural estimation.  In
  particular, we consider the scenario in which the model evolves in a
  piece-wise constant fashion. We propose a procedure that minimizes
  the so-called TESLA loss (i.e., temporally smoothed L1 regularized
  regression), which allows jointly estimating the partition
  boundaries of the VCVS model and the coefficient of the sparse
  precision matrix on each block of the partition. A highly scalable
  proximal gradient method is proposed to solve the resultant convex
  optimization problem; and the conditions for sparsistent estimation
  and the convergence rate of both the partition boundaries and the
  network structure are established for the first time for such
  estimators.
\end{abstract}

\section{Introduction}

Networks are a fundamental form of representation of relational
information underlying large, noisy data from various domains. For
example, in a biological study, nodes of a network can represent genes
in one organism and edges can represent associations or regulatory
dependencies among genes. In a social analysis, nodes of a network can
represent actors and edges can represent interactions or friendships
between actors. Exploring the statistical properties and hidden
characteristics of network entities, and the stochastic processes
behind temporal evolution of network topologies is essential for
computational knowledge discovery and prediction based on network
data.

In many dynamical environments, such as a developing biological
system, it is often technically impossible to experimentally determine
the network topologies specific to every time point in a discrete time
series. Resorting to computational inference methods, such as extant
structural learning algorithms, is also difficult because for every
model unique to a single time point, there exist as few as only a
single snapshot of the nodal states distributed accordingly to the
model in question. In this paper, we consider an estimation problem
under a particular dynamic context, where the model evolves piecewise
constantly, i.e., staying structurally invariant during unknown
segments of time, and then jump to a different structure.

A popular technique for deriving the network structure from {\it iid}
sample is to estimate a sparse precision matrix. The importance of
estimating precision matrices with zeros was recognized by
\citep{dempster72covariance} who coined the term {\it covariance
  selection}.  The elements of the precision matrix represent the
associations or conditional covariances between corresponding
variables. Once a sparse precision matrix is estimated, a network can
be drawn by connecting variables whose corresponding elements of the
precision matrix are non-zero. Recent studies have shown that
covariance selection methods based on the penalized likelihood
maximization can lead to a consistent estimate of the network
structure underlying a Gaussian Markov Random
Fields~\citep{fan09network,ravikumar08high}.  Moreover, a particular
procedure for covariance selection known as neighborhood selection,
which is built on $\ell_1$ norm regularized regression, can produce a
consistent estimate of the network structure when the sample is
assumed to follow a general Markov Random Field distribution whose
structure corresponds to the network in
question~\citep{ravikumar09high,meinshausen06high,peng09partial}.
Specifically, a Markov Random Field (MRF) is a probabilistic graphical
model defined on a graph $G = (V,E)$, where $V = \{1, \ldots, p\}$ is
a vertex set corresponding to the set of random variables to be
modeled (in this paper we call them {\it nodes} and {\it variables}
interchangeably), and $E \subseteq V \times V$ is the edge set
capturing conditional indecencies among these nodes.  Let $\Xb = (X_1,
\ldots, X_p)'$ denote a $p$-dimensional random vector, whose elements
are indexed by the nodes of the graph $G$. Under the MRF, a pair
$(a,b)$ is not an element of the edge set $E$ if and only if the
variable $X_a$ is conditionally independent of $X_b$ given all the
rest of variables $X_{V \backslash \{a,b\}}$, $X_a \perp X_b | X_{V
  \backslash \{a,b\}}$. A distribution over $\Xb$ can be defined by
taking the following log linear form that makes explicit use of the
(presence and absence of edges in the) edge set: $p(\Xb) \propto
\exp\{\sum_{(a,b) \in V} \theta_{ab} X_a X_b \}$.  When the elements
of the random vector $\Xb$ are discrete, e.g., $\Xb \in \{0, 1\}^p$,
the model is referred to as a discrete MRF, sometimes known as an
Ising model in statistics physics community; whereas when $\Xb$ is a
continuous vector, the model is referred to as a Gaussian graphical
model (GGM) because one can easily show that the $p(\Xb)$ above is
actually a multivariate Gaussian.  The MRF have been widely used for
modeling data with graphical relational structures over a fixed set of
entities \citep{wainwright08graphical, getoor07introduction}.  The
vertices can describe entities such as genes in a biological
regulatory network, stocks in the market, or people in society; while
the edges can describe relationships between vertices, for example,
interaction, correlation or influence.

The statistical problem we concern in this paper is to estimate the
structure of the Gaussian graphical model from observed samples of
nodal states in a dynamic world. Traditional methods handle this
problem with the assumption that the samples are {\it iid}. Let $\Dcal
= \{ \xb_1, \ldots, \xb_n\}$ be an {\it independent and identically
  distributed} sample according to a $p$-dimensional multivariate
normal distribution $\Ncal_p(\0b, \Sigmab)$, where $\Sigmab$ is the
covariance matrix. Let $\Omegab := \Sigmab^{-1}$ denote the precision
matrix, with elements $(\omega_{ab})$, $1 \leq a,b \leq p$. Then one
can obtain an estimator of the $\Omegab$ from $\Dcal$ via optimizing a
proper statistical loss function, such as likelihood or penalized
likelihood. As mentioned earlier, the precision matrix $\Omegab$
encodes the conditional independence structure of the distribution and
the pattern of the zero elements in the precision matrix define the
structure of the associated graph $G$.  There has been a dramatic
growth of interest in recent literature in the problem of covariance
selection, which deals with the graph estimation problem
above. Existing works range from algorithmic development focusing on
efficient estimation procedures, to theoretical analysis focusing on
statistical guarantees of different estimators. We do not intend to
give an extensive overview of the literature here, but interested
readers can follow the pointers bellow. In the classical
literature~\citep[e.g.][]{lauritzen96graphical}, procedures are
developed for small dimensional graphs and commonly involve hypothesis
testing with greedy selection of edges. More recent literature
estimates the sparse precision matrix by optimizing penalized
likelihood \citep{yuan07model,fan09network, banerjee08model,
  rothman08spice, friedman08sparse, ravikumar08high,
  guo10jointBiometrika, Zhou08time} or through neighborhood selection
\citep{meinshausen06high, peng09partial,guo10joint,wang09learning},
where the structure of the graph is estimated by estimating the
neighborhood of each node. Both of these approaches are suitable for
high-dimensional problems, even when $p \gg n$, and can be efficiently
implemented using scalable convex program solvers.

Most of the above mentioned work assumes that a single invariant
network model is sufficient to describe the dependencies in the
observed data. However, when the observed data are not {\it iid}, such
an assumption is not justifiable. For example, when data consist of
microarray measurements of the gene expression levels collected
throughout the cell cycle or development of an organism, different
genes can be active during different stages. This suggests that
different distributions and hence different networks should be used to
describe the dependencies between measured variables at different time
intervals. In this paper, we are going to tackle the problem of
estimating the structure of the GGM when the structure is allowed to
change over time. By assuming that the parameters of the precision
matrix change with time, we obtain extra flexibility to model a larger
class of distributions while still retaining the interpretability of
the static GGM. In particular, as the coefficients of the precision
matrix change over time, we also allow the structure of the underlying
graph to change as well. This semi-parametric generalization of the
parametric model is referred to as a varying coefficient varying
structure (VCVS) model.

Now, let $\{\xb_i\}_{i \in [n]} \in \RR^p$ be a sequence of $n$
independent observations (we use $[n]$ to denote the set $\{1, \ldots,
n\}$) from some $p$-dimensional multivariate normal distributions, not
necessarily the same for every observation. Let $\{ \Bcal^j \}_{j \in
  [B]}$ be a disjoint partitioning of the set $[n]$ where each block
of the partition consists of consecutive elements, that is, $\Bcal^j
\cap \Bcal^{j'} = \emptyset$ for $j \neq j'$ and $\bigcup_j \Bcal^j =
[n]$ and $\Bcal^j = [T_{j-1}:T_{j}] := \{T_{j-1}, T_{j-1} + 1, \ldots,
T_{j} - 1\}$. Let $\Tcal := \{ T_0 = 1 < T_1 < \ldots < T_B = n+1\}$
denote the set of partition boundaries. We consider the following
model
\begin{equation}
  \label{eq:model}
  \xb_i \sim \Ncal_p(\mathbf{0}, \Sigmab^j), \qquad i \in \Bcal^j,
\end{equation}
so that observations indexed by elements in $\Bcal^j$ are
$p$-dimensional realizations of a multivariate normal distribution
with zero mean and the covariance matrix $\Sigmab^j =
(\sigma_{ab}^j)_{a,b \in [p]}$. Let $\Omegab^j := (\Sigmab^j)^{-1}$
denote the precision matrix with elements
$(\omega_{ab}^j)_{a,b\in[p]}$.  With the number of partitions, $B$,
and the boundaries of partitions, $\Tcal$, unknown, we study the
problem of estimating both the partition set $\{ \Bcal^j \}$ and the
non-zero elements of the precision matrices $\{\Omegab^j\}_{j \in
  [B]}$ from the sample $\{\xb_i\}_{i \in [n]}$. Note that in this
work we study a particular case of the VCVS model, where the
coefficients are piece-wise constant functions of time. A scenario
where the coefficients are smoothly varying functions of time has been
considered in \cite{Zhou08time} for the GGM and in
\cite{kolar08estimating} and \cite{kolar09sparsistent} for an Ising
model.

If the partitions $\{\Bcal^j\}_j$ were known, the problem would be
trivially reduced to the setting analyzed in the previous
work. Dealing with the unknown partitions, together with the structure
estimation of the model, calls for new methods. We propose and analyze
a method based on {\it time-coupled neighborhood selection}, where the
model estimates are forced to stay similar across time using a
fusion-type total variation penalty and the sparsity of each
neighborhood is obtained through the $\ell_1$ penalty. Details of the
approach are given in $\S2$.

The model in Eq. \eqref{eq:model} is related to the varying-coefficient
models \citep[e.g.][]{hastie93varying} with the coefficients being
piece-wise constant functions. Varying coefficient regression models
with piece-wise constant coefficients are also known as segmented
multivariate regression models \citep{liu97segmented} or linear models
with structural changes \citep{Bai98estimating}. The structural changes
are commonly determined through hypothesis testing and a separate
linear model is fit to each of the estimated segments. In our work, we
use the penalized model selection approach to jointly estimate the
partition boundaries and the model parameters.

Little work has been done so far towards modeling dynamic networks and
estimating changing precision matrices. \cite{Zhou08time} develops a
nonparametric method for estimation of time-varying GGM, where $\xb^t
\sim \Ncal_p(\0b, \Sigmab(t))$ and $\Sigmab(t)$ is smoothly changing
over time. The procedure is based on the penalized likelihood approach
of \cite{yuan07model} with the empirical covariance matrix obtained
using a kernel smoother. Our work is very different from the one of
\cite{Zhou08time}, since under our assumptions the network changes
abruptly rather than smoothly. Furthermore, as we outline in
$\S$\ref{sec:model}, our estimation procedure is not based on the
penalized likelihood approach. Estimation of time-varying Ising models
has been discussed in \cite{amr09tesla} and
\cite{kolar08estimating}. \cite{yin08nonparametric} and
\cite{kolar10nonparametric} studied nonparametric ways to estimate the
conditional covariance matrix. The work of \cite{amr09tesla} is most
similar to our setting, where they also use a fused-type penalty
combined with an $\ell_1$ penalty to estimate the structure of the
verying Ising model. Here, in addition to focusing on GGMs, there is
an additional subtle, but important, difference to
\cite{amr09tesla}. In this work, we use a modification of the fusion
penalty (formally described in $\S$\ref{sec:model}) which allows us to
characterize the model selection consistency of our estimates and the
convergence properties of the estimated partition boundaries, which is
not available in the earlier work.

The remaining of the paper is organized as follows. In
$\S$\ref{sec:model}, we describe our estimation procedure and provide
an efficient first-order optimization procedure capable of estimating
large graphs. The optimization procedure is based on the smoothing
procedure of \cite{nesterov05smooth} and converges in
$\Ocal(1/\epsilon)$ iterations, where $\epsilon$ is the desired
accuracy. Our main theoretical results are presented in
$\S$\ref{sec:theoretical-results}. In particular, we show that the
partition boundaries are estimated consistently. Furthermore, the
graph structure is consistently estimated on every block of the
partition that contains enough samples. Numerical studies showing the
finite sample performance of our procedure are given in
$\S$\ref{sec:numerical-studies}. The proofs of the main results are
relegated to $\S$\ref{sec:proofs}, with some technical details
presented in Appendix.

\subsection*{Notation Schemes:}
For clarity, we end the introduction with a summary of the notations
used in the paper. We use $[n]$ to denote the set $\{1, \ldots, n\}$
and $[l:r]$ to denote the set $\{l, l+1, \ldots, r-1\}$. We use
$\Bcal^j$ to denote $j$-th block of the partition $\Tcal$. With some
abuse of notation, we also use $\Bcal^j$ to denote the set
$[T_{j-1}:T_j]$. The number of samples in the block $\Bcal^j$ is
denoted as $|\Bcal^j|$. For a set $S \subset V$, we use the notation
$X_S$ to denote the set $\{ X_a : a \in S\}$ of random variables. We
use $\Xb$ to denote the $n \times p$ matrix whose rows consist of
observations. The vector $\Xb_a = (x_{1,a}, \ldots, x_{n,a})'$ denotes
a column of matrix $\Xb$ and, similarly, $\Xb_S = (\Xb_b\ :\ b \in S)$
denotes the $n \times |S|$ sub-matrix of $\Xb$ whose columns are
indexed by the set $S$ and $\Xb^{\Bcal^j}$ denotes the sub-matrix
$|\Bcal^j| \times p$ whose rows are indexed by the set
$\Bcal^j$. For simplicity of notation, we will use $\bks a$ to
denote the index set $[p]\bks \{a\}$, $\Xb\ba = (\Xb_b\ :\ b \in
[p]\bks\{a\})$. For a vector $\ab \in \RR^p$, we let $S(\ab)$ denote
the set of non-zero components of $\ab$. Throughout the paper, we use
$c_1, c_2, \ldots$ to denote positive constants whose value may change
from line to line. For a vector $\ab \in \RR^n$, define $\norm{\ab}_1
= \sum_{i \in [n]} |a_i|$, $\norm{\ab}_2 = \sqrt{\sum_{i \in [n]}
  a_i^2}$ and $\norm{\ab}_{\infty} = \max_i |a_i|$. For a symmetric
matrix $\Ab$, $\Lambda_{\min}(\Ab)$ denotes the smallest and
$\Lambda_{\max}(\Ab)$ the largest eigenvalue. For a matrix $\Ab$ (not
necessarily symmetric), we use $\opnorm{\Ab}{\infty} = \max_{i} \sum_j
|A_{ij}|$. For two vectors $\ab, \bb \in \RR^n$, the dot product is
denoted $\dotp{\ab}{\bb} = \sum_{i \in [n]} a_ib_i$. For two matrices
$\Ab, \Bb \in \RR^{n \times m}$, the dot product is denoted as
$\dotpm{\Ab}{\Bb} = \tr( \Ab'\Bb )$. Given two sequences $\{a_n\}$ and
$\{b_n\}$, the notation $a_n = \Ocal(b_n)$ means that there exists a
constant $c_1$ such that $a_n \leq c_1 b_n$; the notation $a_n =
\Omega(b_n)$ means that there exists a constant $c_2$ such that $a_n
\geq c_2 b_n$ and the notation $a_n \asymp b_n$ means that $a_n =
\Ocal(b_n)$ and $b_n = \Ocal(a_n)$. Similarly, we will use the
notation $a_n = o_p(b_n)$ to denote that $b_n^{-1}a_n$ converges to
$0$ in probability.

\begin{table*}
\caption{Summary of symbols used throughout the paper}
\begin{tabular}{c|c|c}
Symbol & Meaning & Example \\
\hline
$[n]$ & used to denote the set $\{1, \ldots, n\}$ 
    & \\
$[t_1:t_2]$ & used to denote the set $\{t_1, t_{1}+1, \ldots, t_{2}-1\}$ 
    & \\
$i$ & used for indexing related to samples 
    & $\xb^i$ or $\beta_{\cdot, i}^a$ \\
$j, k$ & used for indexing related to block 
    & $\thetab^{a,j}$ or $S_a^k$ \\
$a, b$ & used for indexing nodes in a graph
    & $a, b \in V$ \\
$G$ & the graph consisting of vertices and edges  
    & $G = (V,E)$ \\
$V$ & the set of nodes in a graph 
    & $V = [p]$ \\
$E_i$ & the set of edges at time $i$ 
    & \\
$X_a$ & the component of a random vector $\Xb$ indexed by 
    the vertex $a$ & \\
$\betab_{\cdot, i}^a$ & the vector of regression coefficients
    for sample $i$ & \\
$\thetab^{a,j}$ & the vector of regression coefficients 
    for block $j$ & \\
$\Tcal$ & the set of partition boundaries 
    & \\
$\{\tau_j\}_j$ & the set of boundary fractions 
    & $T_j = \lfloor n\tau_j  \rfloor $ \\
$\Bcal^j$ & an index set for the samples in the partition $j$ 
    &  $\Bcal^j \subset [n]$ \\
$B$ & denotes the number of partitions 
    & \\
$S_a^j$ & the set of neighbors of node $a$ in block $j$ 
    & \\
$S(\thetab^{a,j})$ & the set of non-zero elements of $\thetab^{a,j}$ 
    & \\
$\bar S_a^j$ & the closure of $S_a^j$ 
    & $\bar S_a^j = S_a^j \cup \{a\}$ \\
$N_a^j$ & nodes not in the neighborhood of the node $a$ in block $j$
    & $N_a^j = [p] \bks \bar S_a^j$ \\
$\bks a$ & the set of all vertices excluding the vertex $a$ 
    & $\bks a = [p] \bks \{a\}$ \\
$|\cdot|$ & cardinality of a set or absolute value 
    & \\
$\Sigmab$ & the covariance matrix  & \\
$\sigma_{ab}$ & an element of the covariance matrix & \\
$\Omegab$ & the precision matrix & \\
$\omega_{ab}$ & an element of the precision matrix & \\
$\dotp{\cdot}{\cdot}$ & the dot product 
    &  $\dotp{\ab}{\bb} = \ab'\bb$ \\
$\dotpm{\cdot}{\cdot}$ & the dot product between matrices 
   & $\dotpm{\Ab}{\Bb} = \tr(\Ab'\Bb)$ \\
$\xi_{\min}$ & the minimum change between regression coefficient 
   & $\norm{\thetab^{a,j} - \thetab^{a,j-1}}_2 \geq \xi_{\min} $\\
$\theta_{\min}$ & the minimum size of a coefficient 
   & $|\theta_b^{a,j}| \geq \theta_{\min}$ \\
$\Delta_{\min}$ & the minimum size of a block 
   & $|\Bcal^j| \geq \Delta_{\min}$ \\
\hline
\end{tabular}
\end{table*}

\section{Graph estimation via Temporal-Difference
  Lasso} \label{sec:model}

In this section, we introduce our time-varying covariance selection
procedure, which is based on the time-coupled neighborhood selection
using the fused-type penalty. We call the proposed procedure
Temporal-Difference Lasso ({\it TD-Lasso}). We start by reviewing the
basic neighborhood selection procedure, which has previously been used
to estimate graphs in, for example, \cite{peng09partial},
\cite{meinshausen06high}, \cite{ravikumar09high} and \cite{guo10joint}.

We start by relating the elements of the precision matrix $\Omegab$ to
a regression problem. Let the set $S_a$ to denote the neighborhood of
the node $a$. Denote $\bar S_a$ the closure of $S_a$, $\bar S_a := S_a
\cup \{a\}$, and $N_a$ the set of nodes not in the neighborhood of the
node $a$, $N_a = [p] \bks \bar S_a$. It holds that $X_a \perp X_{N_a}
| X_{S_a}$. The neighborhood of the node $a$ can be easily seen from
the non-zero pattern of the elements in the precision matrix
$\Omegab$, $S_a = \{b \in [p]\bks \{a\}\ :\ \omega_{ab} \neq 0\}$. See
\cite{lauritzen96graphical} for more details.  It is a well known
result for Gaussian graphical models that the elements of
\begin{equation*}
  \thetab^a = \argmin_{\thetab \in \RR^{p-1}}\ \EE(X_a - \sum_{b \in
    \bks a} X_b \theta_b)^2
\end{equation*}
are given by $\theta^a_b = - \omega_{ab}/\omega_{aa}$. Therefore, the
neighborhood of a node $a$, $S_a$, is equal to the set of non-zero
coefficients of $\thetab^a$. Using the expression for $\thetab^a$, we
can write $X_a = \sum_{b \in S_a} X_b\theta^a_b + \epsilon$, where
$\epsilon$ is independent of $X\ba$.

The neighborhood selection procedure was motivated by the above
relationship between the regression coefficients and the elements of
the precision matrix. \cite{meinshausen06high} proposed to solve the
following optimization procedure
\begin{equation}
  \label{eq:neighborhood-selection}
  \hat \thetab^a = \argmin_{\thetab \in \RR^{p-1}}\ \frac{1}{n}\norm{\Xb_a -
    \Xb\ba\thetab}_2^2 + \lambda\norm{\thetab}_1
\end{equation}
and proved that for {\it iid} sample the non-zero coefficients of
$\hat \thetab^a$ consistently estimate the neighborhood of the node
$a$, under a suitably chosen penalty parameter $\lambda$.

In this paper, we build on the neighbourhood selection procedure to
estimate the changing graph structure in model~\eqref{eq:model}. We
use $S_a^j$ to denote the neighborhood of the node $a$ on the block
$\Bcal^j$ and $N_a^j$ to denote nodes not in the neighborhood of the
node $a$ on the $j$-th block, $N_a^j = V \bks S_a^j$. Consider the
following estimation procedure
\begin{equation}
  \label{eq:penalized-estimation}
  \hat \betab^a = \argmin_{\betab \in \RR^{p-1 \times n}}\ \Lcal(\betab) +
  \pen_{\lambda_1, \lambda_2}(\betab)
\end{equation}
where the loss is defined for $\betab = (\beta_{b,i})_{b \in [p-1], i
  \in [n]}$ as
\begin{equation}
  \label{eq:loss}
  \Lcal(\betab) := \sum_{i \in [n]} \bigg(
        x_{i,a} - \sum_{b \in \bks a} x_{i,b} \beta_{b,i}
      \bigg)^2
\end{equation}
and the penalty is defined as 
\begin{equation}
  \label{eq:penalty}
  \pen_{\lambda_1, \lambda_2}(\betab) := 
    2\lambda_1 \sum_{i=2}^{n} \norm{\betab_{\cdot, i} - \betab_{\cdot,i-1}}_2
  + 2\lambda_2\sum_{i=1}^n \sum_{b \in \bks a} |\beta_{b,i}|.
\end{equation}
The penalty term is constructed from two terms. The first term ensures
that the solution is going to be piecewise constant for some partition
of $[n]$ (possibly a trivial one). The first term can be seen as a
sparsity inducing term in the temporal domain, since it penalizes the
difference between the coefficients $\betab_{\cdot, i}$ and
$\betab_{\cdot, i+1}$ at successive time-points. The second term
results in estimates that have many zero coefficients within each
block of the partition. The estimated set of partition boundaries
\begin{equation*}
  \hat \Tcal = \{\hat T_0 = 1\} \cup \{ \hat T_j \in [2:n]\ : \ \hat
  \betab_{\cdot, \hat T_j}^a
  \neq \hat \betab_{\cdot,\hat T_j-1}^a\}  \cup \{\hat T_{\hat B} = n+1\}
\end{equation*}
contains indices of points at which a change is estimated, with $\hat
B$ being an estimate of the number of blocks $B$. The estimated number
of the block $\hat B$ is controlled through the user defined penalty
parameter $\lambda_1$, while the sparsity of the neighborhood is
controlled through the penalty parameter $\lambda_2$.

Based on the estimated set of partition boundaries $\hat \Tcal$, we
can define the neighborhood estimate of the node $a$ for each
estimated block. Let $\hat \thetab^{a,j} = \hat \betab^a_{\cdot, i}$,
$\forall i \in [\hat T_{j-1}:\hat T_j]$ be the estimated coefficient
vector for the block $\hat \Bcal^j = [\hat T_{j-1}:\hat T_j]$. Using
the estimated vector $\hat \thetab^{a,j}$, we define the neighborhood
estimate of the node $a$ for the block $\hat \Bcal^j$ as
\begin{equation*}
  \hat S_a^j := S(\hat \thetab^{a,j}) 
            := \{ b \in \bks a\ :\ \hat \theta^{a,j}_b \neq 0\}.
\end{equation*}
Solving \eqref{eq:penalized-estimation} for each node $a \in V$ gives
us a neighborhood estimate for each node. Combining the neighborhood
estimates we can obtain an estimate of the graph structure for each
point $i \in [n]$. 

The choice of the penalty term is motivated by the work on
penalization using total variation \citep{rinaldo09properties,
  mammen97locally}, which results in a piece-wise constant approximation
of an unknown regression function. The fusion-penalty has also been
applied in the context of multivariate linear regression
\cite{tibshirani05sparsity}, where the coefficients that are spatially
close, are also biased to have similar values. As a result, nearby
coefficients are fused to the same estimated value. Instead of
penalizing the $\ell_1$ norm on the difference between coefficients,
we use the $\ell_2$ norm in order to enforce that all the changes
occur at the same point. 

The objective \eqref{eq:penalized-estimation} estimates the
neighborhood of one node in a graph for all time-points. After solving
the objective \eqref{eq:penalized-estimation} for all nodes $a \in V$,
we need to combine them to obtain the graph structure. We will use the
following procedure to combine $\{ \hat \betab^a \}_{a \in V}$,
\begin{equation*}
  \hat E_i = \{ (a, b)\ :\ \max(|\beta^a_{b, i}|, |\beta^b_{a, i}|) >
  0 \}, \qquad i \in [n].
\end{equation*}
That is, an edge between nodes $a$ and $b$ is included in the graph if
at least one of the nodes $a$ or $b$ is included in the neighborhood
of the other node. We use the $\max$ operator to combine different
neighborhoods as we believe that for the purpose of network
exploration it is more important to occasionally include spurious
edges than to omit relevant ones. For further discussion on the
differences between the min and the max combination, we refer an
interested reader to \cite{banerjee08model}.

\subsection{Numerical procedure}

Finding a minimizer $\hat \betab^a$ of \eqref{eq:penalized-estimation}
can be a computationally challenging task for an off-the-shelf convex
optimization procedure. We propose too use an accelerated gradient
method with a smoothing technique \citep{nesterov05smooth}, which
converges in $\Ocal(1/\epsilon)$ iterations where $\epsilon$ is the
desired accuracy. 

We start by defining a smooth approximation of the fused penalty
term. Let $\Hb \in \RR^{n \times n-1}$ be a matrix with elements
\begin{equation*}
  H_{ij} = \left\{ 
  \begin{array}{cl}
    -1 & \text{if } i = j\\
     1 & \text{if } i = j + 1\\
     0 & \text{otherwise.}
  \end{array}\right.
\end{equation*}
With the matrix $\Hb$ we can rewrite the fused penalty term as
$2\lambda_1\sum_{i=1}^{n-1}\norm{(\betab \Hb)_{\cdot, i}}_2$ and using
the fact that the $\ell_2$ norm is self dual \citep[for example,
see][]{boyd04convex} we have the following representation
\begin{equation}
  \label{eq:dual-norm}
  2\lambda_1\sum_{i=2}^{n}\norm{\betab_{\cdot,i}-\betab_{\cdot,i-1}}_2
  = 
  \max_{\Ub \in \Qcal}\ \dotpm{\Ub}{2\lambda_1 \betab \Hb}
\end{equation}
where $\Qcal := \{ \Ub \in \RR^{p-1 \times n-1}\ :\ \norm{\Ub_{\cdot,
    i}}_2 \leq 1,\ \forall i \in [n-1] \}$. The following function is
defined as a smooth approximation to the fused penalty,
\begin{equation}
  \label{eq:smooth-approx}
  \Psi_{\mu}(\betab) :=  \max_{\Ub \in \Qcal}\ \dotpm{\Ub}{2\lambda_1
    \betab \Hb} - \mu \norm{\Ub}_F^2
\end{equation}
where $\mu > 0$ is the smoothness parameter. It is easy to see that
\begin{equation*}
  \Psi_{\mu}(\betab) \leq \Psi_0(\betab)
  \leq\Psi_{\mu}(\betab) + \mu(n-1).
\end{equation*}
Setting the smoothness parameter to $\mu =
\smallfrac{\epsilon}{2(n-1)}$, the correct rate of convergence is
ensured.  Let $\Ub_\mu(\betab)$ be the optimal solution of the
maximization problem in \eqref{eq:smooth-approx}, which can be
obtained analytically as
\begin{equation}
  \label{eq:optimal-dual-variable}
  \Ub_\mu(\betab) = \Pi_{\Qcal}\left( \frac{\lambda \betab \Hb}{\mu} \right)
\end{equation}
where $\Pi_{\Qcal}(\cdot)$ is the projection operator onto the set
$\Qcal$. From Theorem 1 in \cite{nesterov05smooth}, we have that
$\Psi_{\mu}(\betab)$ is continuously differentiable and convex, with
the gradient 
\begin{equation}
  \label{eq:gradient-smooth-approx}
  \nabla \Psi_{\mu}(\betab) = 2\lambda_1\Ub_{\mu}(\betab)\Hb'
\end{equation}
that is Lipschitz continuous. 

With the above defined smooth approximation, we focus on minimizing
the following objective 
\begin{equation*}
  \min_{\betab \in \RR^{p-1 \times n}} F(\betab) :=
  \min_{\betab \in \RR^{p-1 \times n}} \Lcal(\betab) +
  \Psi_{\mu}(\betab) + 2\lambda_2 \norm{\betab}_1.
\end{equation*}
Following \cite{beck09fast} (see also \cite{nesterov07gradient}), we
define the following quadratic approximation of $F(\betab)$ at a point
$\betab_0$
\begin{equation}
  \label{eq:quadratic-approximation}
\begin{aligned}
  Q_L(\betab, \betab_0) := \Lcal(\betab_0) & + \Psi_{\mu}(\betab_0)  + 
  \dotpm{\betab-\betab_0}{\nabla \Lcal(\betab_0)+\nabla \Psi(\betab_0)}
  \\ &+ \frac{L}{2} \norm{\betab-\betab_0}_F^2 + 2\lambda_2
  \norm{\betab}_1
\end{aligned}
\end{equation}
where $L > 0$ is the parameter chosen as an upper bounds for the
Lipschitz constant of $\nabla \Lcal + \nabla \Psi$. Let
$p_L(\betab_0)$ be a minimizer of $Q_L(\betab, \betab_0)$. Ignoring
constant terms, $p_L(\betab_0)$ can be obtained as
\begin{equation*}
  \label{eq:minimizer_approximation}
  p_L(\betab_0) = \argmin_{\betab \in \RR^{p-1 \times n}}
  \ \frac{1}{2}
  \bigg\|\betab - \Big(\betab_0 - \frac{1}{L}\big(\nabla \Lcal +
    \nabla \Psi\big)(\betab_0)\Big)\bigg\|_F^2 
  + \frac{2\lambda_2}{L}\big\|\betab\big\|_1.
\end{equation*}
It is clear that $p_L(\betab_0)$ is the unique minimizer, which can be
obtained in a closed form, as a result of the soft-thresholding,
\begin{equation}
  \label{eq:minimizer-approximation-closed-form}
  p_L(\betab_0) = T\bigg(\betab_0 - \frac{1}{L}\big(\nabla \Lcal +
    \nabla \Psi\big)(\betab_0), \frac{2\lambda_2}{L}\bigg)
\end{equation}
where $T(x, \lambda) = \sgn(x)\max(0, |x| - \lambda)$ is the
soft-thresholding operator that is applied element-wise. 

In practice, an upper bound on the Lipschitz constant of $\nabla \Lcal
+ \nabla \Psi$ can be expensive to compute, so the parameter $L$ is
going to be determined iteratively. Combining all of the above, we
have the following algorithm.
\begin{algorithm}[thb]
\normalsize
\SetAlgoSkip{bigskip}
\TitleOfAlgo{Accelerated Gradient Method for
  Equation~\eqref{eq:penalized-estimation}}
\DontPrintSemicolon
\BlankLine
\KwIn{$\Xb \in \RR^{n \times p}$, $\betab_0 \in \RR^{p-1
  \times n}$, $\gamma > 1$, $L > 0$, $\mu =
\frac{\epsilon}{2(n-1)}$}
\KwOut{$\hat{\betab}^{a}$}
\BlankLine
Initialize $k:=1$, $\alpha_k := 1$, $\zb_k := \betab_0$ \;
\Repeat{convergence}{
\While{$F(p_L(\zb_k)) > Q_L(p_L(\zb_k), \zb_k)$}
 {\BlankLine$L := \gamma L$\BlankLine} 
\BlankLine$\betab_k := p_L(\zb_k)$
\qquad (using
  Eq.~\eqref{eq:minimizer-approximation-closed-form})
\;
\BlankLine
$\alpha_{k+1} := \frac{1+\sqrt{1+4\alpha_k}}{2}$ \;\BlankLine
$\zb_{k+1} := \betab_k + \frac{\alpha_k - 1}{\alpha_{k+1}}
\big(\betab_k - \betab_{k-1} \big)$\;\BlankLine
} 
$\hat \betab^a := \betab_k$ \;
\end{algorithm}
In the algorithm, $\gamma$ is a constant used to increase the estimate
of the Lipschitz constant $L$. Compared to the gradient descent method
(which can be obtain by iterating $\betab_{k+1} = p_L(\betab_k)$), the
accelerated gradient method updates two sequences $\{\betab_k\}$ and
$\{\zb_k\}$ recursively. Instead of performing the gradient step from
the latest approximate solution $\betab_k$, the gradient step is
performed from the search point $\zb_k$ that is obtained as a linear
combination of the last to approximate solutions $\betab_{k-1}$ and
$\betab_k$. Since the condition $F(p_L(\zb_k)) \leq Q_L(p_L(\zb_k),
\zb_k)$ is satisfied in every iteration, we have the algorithm
converges in $\Ocal(1/\epsilon)$ iterations following
\cite{beck09fast}. As the convergence criterion, we stop iterating
once the relative change in the objective value is below some
threshold value. 

\subsection{Tuning parameter selection}
\label{sec:bic-score}

The penalty parameters $\lambda_1$ and $\lambda_2$ control the
complexity of the estimated model. In this work, we propose to use the
BIC score to select the tuning parameters. Define the BIC score for
each node $a \in V$ as 
\begin{equation}
  \label{eq:bic-score}
  {\rm BIC}_a(\lambda_1, \lambda_2) := 
    \log \frac{\Lcal(\hat \betab^a)}{n}
    + \frac{\log n}{n} \sum_{j \in [\hat B]} |S(\hat \thetab^{a,j})|
\end{equation}
where $\Lcal(\cdot)$ is defined in~\eqref{eq:loss} and $\hat \betab^a
= \hat \betab^a(\lambda_1, \lambda_2)$ is a solution
of~\eqref{eq:penalized-estimation}. The penalty parameters can now be
chosen as 
\begin{equation}
  \label{eq:penalty-parameters-bic-chosen}
  \{\hat \lambda_1, \hat \lambda_2\} = \argmin_{\lambda_1, \lambda_2} 
      \sum_{a \in V} {\rm BIC}_a(\lambda_1, \lambda_2).
\end{equation}
We will use the above formula to select the tuning parameters in our
simulations, where we are going to search for the best choice of
parameters over a grid.

\section{Theoretical results} \label{sec:theoretical-results}

This section is going to address the statistical properties of the
estimation procedure presented in Section~\ref{sec:model}. The
properties are addressed in an asymptotic framework by letting the
sample size $n$ grow, while keeping the other parameters fixed. For
the asymptotic framework to make sense, we assume that there exists a
fixed unknown sequence of numbers $\{\tau_j\}$ that defines the
partition boundaries as $T_j = \lfloor n\tau_j \rfloor$, where
$\lfloor a \rfloor$ denotes the largest integer smaller that $a$. This
assures that as the number of samples grow, the same fraction of
samples falls into every partition. We call $\{\tau_j\}$ the boundary
fractions.

We give sufficient conditions under which the sequence $\{ \tau_j\}$
is consistently estimated. In particular, if the number of partition
blocks is estimated correctly, then we show that $\max_{j \in [B]}
|\hat T_j - T_j| \leq n\delta_n$ with probability tending to 1, where
$\{\delta_n\}_n$ is a non-increasing sequence of positive numbers that
tends to zero. If the number of partition segments is over estimated,
then we show that for a distance defined for two sets $A$ and $B$ as
\begin{equation}
  \label{eq:distance-sets}
  h(A, B) := \sup_{b \in B} \inf_{a \in A} |a - b|,
\end{equation}
we have $h(\hat \Tcal, \Tcal) \leq n\delta_n$ with probability tending
to 1. With the boundary segments consistently estimated, we further
show that under suitable conditions for each node $a \in V$ the
correct neighborhood is selected on all estimated block partitions
that are sufficiently large. 

The proof technique employed in this section is quite involved, so we
briefly describe the steps used. Our analysis is based on careful
inspection of the optimality conditions that a solution $\hat
\betab^a$ of the optimization problem~\eqref{eq:penalized-estimation}
need to satisfy. The optimality conditions for $\hat \betab^a$ to be a
solution of~\eqref{eq:penalized-estimation} are given in
$\S$\ref{sec:conv-part-bound}.  Using the optimality conditions, we
establish the rate of convergence for the partition boundaries. This
is done by proof by contradiction. Suppose that there is a solution
with the partition boundary $\hat \Tcal$ that satisfies $h(\hat \Tcal,
\Tcal) \geq n\delta_n$. Then we show that, with high-probability, all
such solutions will not satisfy the KKT conditions and therefore
cannot be optimal. This shows that all the solutions to the
optimization problem \eqref{eq:penalized-estimation} result in
partition boundaries that are ``close'' to the true partition
boundaries, with high-probability. Once it is established that $\hat
\Tcal$ and $\Tcal$ satisfy $h(\hat \Tcal, \Tcal) \leq n\delta_n$, we
can further show that the neighborhood estimates are consistently
estimated, under the assumption that the estimated blocks of the
partition have enough samples. This part of the analysis follows the
commonly used strategy to prove that the Lasso is sparsistent
\citep[see for
example][]{bunea08honest,wainwright06sharp,meinshausen06high}, however
important modifications are required due to the fact that position of
the partition boundaries are being estimated.

Our analysis is going to focus on one node $a \in V$ and its
neighborhood. However, using the union bound over all nodes in $V$, we
will be able to carry over conclusions to the whole graph. To simplify
our notation, when it is clear from the context, we will omit the
superscript $a$ and write $\hat \betab$, $\hat \thetab$ and $S$, etc.,
to denote $\hat \betab^a$, $\hat \thetab^a$ and $S_a$, etc.

\subsection{Assumptions}
\label{sec:assumptions}

Before presenting our theoretical results, we give some definitions
and assumptions that are going to be used in this section. Let
$\Delta_{\min} := \min_{j \in [B]} |T_j - T_{j-1}|$ denote the minimum
length between change points, $\xi_{\min} := \min_{a \in V} \min_{j
  \in [B-1]} \norm{\thetab^{a,j+1} - \thetab^{a,j}}_2$ denote the
minimum jump size and $\theta_{\min} = \min_{a \in V} \min_{j \in [B]}
\min_{b \in S^j} |\theta_b^{a,j}|$ the minimum coefficient size.
Throughout the section, we assume that the following holds.

\begin{description}
\item[\bf A1] There exist two constants $\phi_{\min} > 0$ and
  $\phi_{\max} < \infty$ such that
  \begin{equation*}
    \phi_{\min} = \min\ \{\Lambda_{\min}(\Sigmab^j)\ :\ j
    \in [B], a \in V\}
  \end{equation*}
  and
  \begin{equation*}
    \phi_{\max} = \max\ \{\Lambda_{\max}(\Sigmab^j)\ :\ j
    \in [B], a \in V\}.
  \end{equation*}

\item[\bf A2] Variables are scaled so that $\sigma_{aa}^j = 1$ for all
  $j \in [B]$ and all $a \in V$.
\end{description}
The assumption {\bf A1} is commonly used to ensure that the model is
identifiable. If the population covariance matrix is ill-conditioned,
the question of the correct model identification if not well defined,
as a neighborhood of a node may not be uniquely defined. The
assumption {\bf A2} is assumed for the simplicity of the
presentation. The common variance can be obtained through scaling.

\begin{description}
\item[\bf A3] There exists a constant $M > 0$ such that 
  \begin{equation*}
    \max_{a \in V}\max_{j,k \in [B]} \| \thetab^{a,k} - \thetab^{a,j} \|_2 \leq M.    
  \end{equation*}
\end{description}
The assumption {\bf A3} states that the difference between
coefficients on two different blocks, $\norm{\thetab^{a,k} -
  \thetab^{a,j}}_2$, is bounded for all $j, k \in [B]$. This
assumption is simply satisfied if the coefficients $\thetab^a$ were
bounded in the $\ell_2$ norm.

\begin{description}
\item[\bf A4] There exist a constant $\alpha \in (0, 1]$, such that
  the following holds
  \begin{equation*}
    \max_{j \in [B]} \opnorm{ \Sigmab_{N_a^j S_a^j} 
      ( \Sigmab_{S_a^j S_a^j})^{-1}}{\infty} \leq 1 - \alpha,
    \qquad \forall a \in V.
  \end{equation*} 
\end{description}
The assumption {\bf A4} states that the variables in the neighborhood
of the node~$a$, $S_a^j$, are not too correlated with the variables in
the set $N_a^j$.  This assumption is necessary and sufficient for
correct identification of the relevant variables in the Lasso
regression problems \citep[see for
example][]{zhao06model,geer09conditions}. Note that this condition is
sufficient also in our case when the correct partition boundaries are
not known.

\begin{description}
\item[\bf A5] The minimum coefficient size $\theta_{\min}$ satisfies
  $\theta_{\min} = \Omega(\sqrt{\log (n) / n})$.
\end{description}
The lower bound on the minimum coefficient size $\theta_{\min}$ is
necessary, since if a partial correlation coefficient is too close to
zero the edge in the graph would not be detectable.

\begin{description}
\item[\bf A6] The sequence of partition boundaries $\{ T_j\}$ satisfy
  $T_j = \lfloor n\tau_j \rfloor$, where $\{ \tau_j \}$ is a fixed,
  unknown sequence of the boundary fractions belonging to $[0,1]$.
\end{description}
The assumption is needed for the asymptotic setting. As $n \rightarrow
\infty$, there will be enough sample points in each of the blocks to
estimate the neighborhood of nodes correctly. 

\subsection{Convergence of the partition boundaries}
\label{sec:conv-part-bound}

In this subsection we establish the rate of convergence of the
boundary partitions for the estimator~\eqref{eq:penalized-estimation}.
We start by giving a lemma that characterizes solutions of the
optimization problem given in \eqref{eq:penalized-estimation}. Note
that the optimization problem in \eqref{eq:penalized-estimation} is
convex, however, there may be multiple solutions to it, since it is
not strictly convex.
\begin{lemma} \label{lem:kkt_conditions} A matrix $\hat \betab$ is
  optimal for the optimization problem \eqref{eq:penalized-estimation}
  if and only if there exist a collection of subgradient vectors
  $\{\hat \zb_i\}_{i \in [2:n]}$ and $\{\hat \yb_i\}_{i \in [n]}$, with
  $\hat \zb_i \in \partial \norm{\hat \betab_{\cdot, i} - \hat
    \betab_{\cdot, i-1}}_2$ and $\hat \yb_i \in
  \partial \norm{\hat \betab_{\cdot, i}}_1$, that satisfies 
  \begin{equation}
    \label{eq:kkt_conditions}    
    \sum_{i=k}^n \xb_{i,\bks a} 
       \dotp{\xb\iba}{\hat \betab_{\cdot, i} - \betab_{\cdot, i}}
    - \sum_{i=k}^n \xb\iba \epsilon_i 
    + \lambda_1 \hat \zb_k + \lambda_2 \sum_{i=k}^n \hat \yb_i = 0
  \end{equation}
  for all $k \in [n]$ and $\hat \zb_1 = \hat \zb_{n+1} = \0b$.
\end{lemma}

The following theorem provides the convergence rate of the estimated
boundaries of $\hat \Tcal$, under the assumption that the correct
number of blocks is known.

\begin{theorem} \label{thm:consistent-estimation-boundaries:known-num-blocks}
  Let $\{\xb_i\}_{i \in [n]}$ be a sequence of observation according
  to the model in \eqref{eq:model}. Assume that {\bf A1}-{\bf A3} and
  {\bf A5}-{\bf A6} hold. Suppose that the penalty parameters
  $\lambda_1$ and $\lambda_2$ satisfy
  \begin{equation}
    \label{eq:scalling-penalty-parameters}
    \lambda_1 \asymp \lambda_2 = \Ocal(\sqrt{ \log(n) / n}).
  \end{equation}
  Let $\{\hat \betab_{\cdot, i}\}_{i \in [n]}$ be any solution of
  \eqref{eq:penalized-estimation} and let $\hat \Tcal$ be the
  associated estimate of the block partition. Let $\{\delta_n\}_{n
    \geq 1}$ be a non-increasing positive sequence that converges to
  zero as $n \rightarrow \infty$ and satisfies $\Delta_{\min} \geq
  n\delta_n$ for all $n \geq 1$. Furthermore, suppose that $(n\delta_n
  \xi_{\min})^{-1} \lambda_1 \rightarrow 0$,
  $\xi_{\min}^{-1}\sqrt{p}\lambda_2 \rightarrow 0$ and
  $(\xi_{\min}\sqrt{n\delta_n})^{-1}\sqrt{p\log n} \rightarrow 0$,
  then if $|\hat \Tcal| = B + 1$ the following holds
  \begin{equation*}
    \label{eq:rate_convergence:change_points}
    \PP[\max_{j \in [B]} |T_j - \hat T_j| \leq n \delta_n] 
     \xrightarrow{n \rightarrow \infty} 1.
  \end{equation*}
\end{theorem}

Suppose that $\delta_n = (\log n)^\gamma / n$ for some $\gamma > 1$
and $\xi_{\min} = \Omega(\sqrt{\log n / (\log n)^\gamma})$, the
conditions of theorem 5 are satisfied, and we have that the sequence
of boundary fractions $\{ \tau_j \}$ is consistently estimated. Since
the boundary fractions are consistently estimated, we will see below
that the estimated neighborhood $S(\hat \thetab^{j})$ on the block
$\hat \Bcal^j$ consistently recovers the true neighborhood $S^j$. 

Unfortunately, the correct bound on the number of block $B$ may not be
known. However, a conservative upper bound $B_{\max}$ on the number of
blocks $B$ may be known. Suppose that the sequence of observation is
over segmented, with the number of estimated blocks bounded by
$B_{\max}$. Then the following proposition gives an upper bound on
$h(\hat \Tcal, \Tcal)$ where $h(\cdot,\cdot)$ is defined
in~\eqref{eq:distance-sets}.

\begin{proposition}
  \label{prop:oversegmented-boundaries}
  Let $\{\xb_i\}_{i \in [n]}$ be a sequence of observation according
  to the model in \eqref{eq:model}. Assume that the conditions of
  theorem~\ref{thm:consistent-estimation-boundaries:known-num-blocks}
  are satisfied. Let $\hat \betab$ be a solution
  of~\eqref{eq:penalized-estimation} and $\hat \Tcal$ the
  corresponding set of partition boundaries, with $\hat B$ blocks. If
  the number of blocks satisfy $B \leq \hat B \leq B_{\max}$, then
  \begin{equation*}
    \PP[h(\hat \Tcal, \Tcal) \leq n \delta_n] 
      \xrightarrow{n \rightarrow \infty} 1.
  \end{equation*}  
\end{proposition}
The proof of the proposition follows the same ideas of
theorem~\ref{thm:consistent-estimation-boundaries:known-num-blocks}
and its sketch is given in the appendix.

The above proposition assures us that even if the number of blocks is
overestimated, there will be a partition boundary close to every true
unknown partition boundary.

\subsection{Correct neighborhood selection}

In this section, we give a result on the consistency of the
neighborhood estimation. We will show that whenever the estimated
block $\hat \Bcal^j$ is large enough, say $|\hat \Bcal^j| \geq r_n$
where $\{r_n\}_{n \geq 1}$ is an increasing sequence of numbers that
satisfy $(r_n\lambda_2)^{-1}\lambda_1 \rightarrow 0$ and
$r_n\lambda_2^2 \rightarrow \infty$ as $n \rightarrow \infty$, we
have that $S(\hat \thetab^j) = S(\betab^k)$, where $\betab^k$ is the
true parameter on the true block $\Bcal^k$ that overlaps $\hat
\Bcal^j$ the most. Figure~\ref{fig:neighborhood-selection} illustrates
this idea. The blue region in the figure denotes the overlap between
the true block and the estimated block of the partition. The orange
region corresponds to the overlap of the estimated block with a
different true block. If the blue region is considerably larger than
the orange region, the bias coming from the sample from the orange
region will not be strong enough to disable us from selecting the
correct neighborhood. On the other hand, since the orange region is
small, as seen from
Theorem~\ref{thm:consistent-estimation-boundaries:known-num-blocks},
there is little hope of estimating the neighborhood correctly on that
portion of the sample.

\begin{figure}[t]
  \centering
  \includegraphics[width=0.7\textwidth]{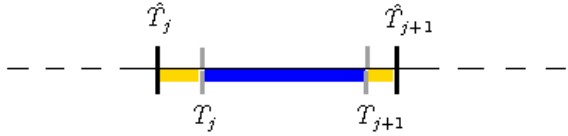}
  \caption{The figure illustrates where we expect to estimate a
    neighborhood of a node consistently. The blue region corresponds
    to the overlap between the true block (bounded by gray lines) and
    the estimated block (bounded by black lines). If the blue region
    is much larger than the orange regions, the additional bias
    introduced from the samples from the orange region will not
    considerably affect the estimation of the neighborhood of a node
    on the blue region. However, we cannot hope to consistently
    estimate the neighborhood of a node on the orange region. }
  \label{fig:neighborhood-selection}
\end{figure}

Suppose that we know that there is a solution to the optimization
problem~\eqref{eq:penalized-estimation} with the partition boundary
$\hat \Tcal$. Then that solution is also a minimizer of the following
objective
\begin{equation}
  \label{eq:objective_known_boundaries}
  \min_{\thetab^{1}, \ldots, \thetab^{\hat B}}\ \
  \sum_{j \in \hat B} 
     \norm{\Xb_a\supb{j} - \Xb\ba\supb{j}\thetab^{j}}_2^2
  + 2\lambda_1 
      \sum_{j = 2}^{\hat B} \norm{\thetab^{j} - \thetab^{j-1}}_2
  + 2\lambda_2 \sum_{j=1}^{\hat B} |\hat \Bcal^j|\norm{\thetab^j}_1.
\end{equation}
Note that the problem \eqref{eq:objective_known_boundaries} does not
give a practical way of solving \eqref{eq:penalized-estimation}, but
will help us to reason about the solutions of
\eqref{eq:penalized-estimation}. In particular, while there may be
multiple solutions to the problem \eqref{eq:penalized-estimation},
under some conditions, we can characterize the sparsity pattern of any
solution that has specified partition boundaries $\hat \Tcal$.

\begin{lemma} \label{lem:unique_pattern} Let $\hat \betab$ be a
  solution to \eqref{eq:penalized-estimation}, with $\hat \Tcal$ being
  an associated estimate of the partition boundaries. Suppose that the
  subgradient vectors satisfy $|\hat y_{i,b}| < 1$ for all $b \not\in
  S(\hat \betab_{\cdot, i})$, then any other solution $\tilde \betab$
  with the partition boundaries $\hat \Tcal$ satisfy $\tilde
  \beta_{b,i} = 0$ for all $b \not\in S(\hat \betab_{\cdot, i})$.
\end{lemma}

The above Lemma states sufficient conditions under which the sparsity
patter of a solution with the partition boundary $\hat \Tcal$ is
unique. Note, however, that there may other solutions to
\eqref{eq:penalized-estimation} that have different partition
boundaries. 

Now, we are ready to state the following theorem, which establishes
that the correct neighborhood is selected on every sufficiently large
estimated block of the partition.
\begin{theorem}
  \label{thm:correct-neighborhood}
  Let $\{\xb_i\}_{i \in [n]}$ be a sequence of observation according
  to the model in \eqref{eq:model}. Assume that the conditions of
  theorem~\ref{thm:consistent-estimation-boundaries:known-num-blocks}
  are satisfied. In addition, suppose that {\bf A4} also holds. Then,
  if $|\hat \Tcal| = B + 1$, it holds that
  \begin{equation*}
    \PP[S^k = S(\hat \thetab^{k})] \xrightarrow{n \rightarrow \infty}
    1, \qquad \forall k \in [B].
  \end{equation*}
\end{theorem}
Under the assumptions of
theorem~\ref{thm:consistent-estimation-boundaries:known-num-blocks}
each estimated block is of size $\Ocal(n)$. As a result, there are
enough samples in each block to consistently estimate the underlying
neighborhood structure. Observe that the neighborhood is consistently
estimated at each $i \in \hat \Bcal^j \cap \Bcal^j$ for all $j \in
[B]$ and the error is made only on the small fraction of samples, when
$i \not\in \hat \Bcal^j \cap \Bcal^j$, which is of order
$\Ocal(n\delta_n)$. 

Using proposition~\ref{prop:oversegmented-boundaries} in place of
theorem~\ref{thm:consistent-estimation-boundaries:known-num-blocks},
it can be similarly shown that, for a large fraction of samples, the
neighborhood is consistently estimated even in the case of
over-segmentation. In particular, whenever there is a sufficiently
large estimated block, with $|\hat \Bcal^k \cap \Bcal^j| =
\Ocal(r_n)$, it holds that $S(\hat \Bcal^k) = S^j$ with probability
tending to one.

\section{Numerical studies}
\label{sec:numerical-studies}

In this section, we present a small numerical study on the proposed
algorithm on simulated networks. A full performance test and
application on real world data is beyond the scope of this paper which
mainly focuses on the theory of time-varying model estimation. In all
of our simulations studies we set $p=30$ and $B = 3$ with $|\Bcal_1| =
80$, $|\Bcal_2| = 130$ and $|\Bcal_3| = 90$, so that in total we have
$n=300$ samples. We consider two types of random networks: a chain and
a nearest neighbor network. We measure the performance of the
estimation procedure outlined in $\S$\ref{sec:model} on the following
metrics: average precision of estimated edges, average recall of
estimated edges and average $F_1$ score which combines the precision
and recall score. The precision, recall and $F_1$ score are
respectively defined as
\begin{equation*}
\begin{aligned}
  precision &= \frac{1}{n}\sum_{i \in [n]} 
    \frac{\sum_{a\in [p]}\sum_{b = a+1}^p 
      \ind\{(a,b) \in \hat E_i\ \wedge (a,b) \in E_i\}}
   {\sum_{a\in [p]}\sum_{b = a+1}^p
     \ind\{(a,b) \in \hat E_i\}} \\
  recall &= \frac{1}{n}\sum_{i \in [n]} 
    \frac{\sum_{a\in [p]}\sum_{b = a+1}^p 
      \ind\{(a,b) \in \hat E_i\ \wedge (a,b) \in E_i\}}
   {\sum_{a\in [p]}\sum_{b = a+1}^p
     \ind\{(a,b) \in E_i\}}\\
  F_1 & = \frac{2 * precision * recall}{precision + recall}.
\end{aligned}
\end{equation*}
Our results are averaged over 50 simulation runs. We compare our
algorithm against an oracle algorithm which exactly knows the true
partition boundaries. In this case, it is only needed to run the
algorithm of \cite{meinshausen06high} on each block of the partition
independently. We use a BIC criterion to select the tuning parameter
for this oracle procedure as described in \cite{peng09partial}.

{\bf Chain networks.} We follow the simulation in \cite{fan09network}
to generate a chain network (see Figure~\ref{fig:chain-graph}). This
network corresponds to a tridiagonal precision matrix (after an
appropriate permutation of nodes). The network is generated as
follows. First, we choose generate a random permutation $\pi$ of
$[n]$. Next, the covariance matrix is generated as follows: the
element at position $(a,b)$ is chosen as $\sigma_{ab} =
\exp(-|t_{\pi(a)} - t_{\pi(b)}|/2)$ where $t_1 < t_2 < \ldots < t_p$
and $t_i - t_{i-1} \sim {\rm Unif}(0.5, 1)$ for $i = 2, \ldots,
p$. This processes is repeated three times to obtain three different
covariance matrices, from which we sample $80$, $130$ and $90$ samples
respectively.

\begin{figure}[t]
  \centering
  \includegraphics[width=0.7\textwidth]{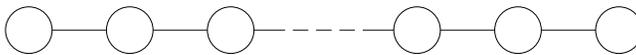}
  \caption{A chain graph}
  \label{fig:chain-graph}
\end{figure}

For illustrative purposes, Figure~\ref{fig:chain} plots the precision,
recall and $F_1$ score computed for different values of the penalty
parameters $\lambda_1$ and $\lambda_2$. Table~\ref{table:chain} shows
the precision, recall and $F_1$ score for the parameters chosen using
the BIC score described in \ref{sec:bic-score}. The numbers in
parentheses correspond to standard deviation. Due to the fact that
there is some error in estimating the partition boundaries, we observe
a decrease in performance compared to the oracle procedure that knows
the correct position of the partition boundaries.

\begin{table}[b]
  \caption{Performance on chain networks}
  \begin{tabular}{c||c|c|c}
    Method name & Precision & Recall & $F_1$ score \\
    \hline
    TD-Lasso & 0.84 (0.04) & 0.80 (0.04) & 0.82 (0.04) \\
    Oracle procedure & 0.97 (0.02) & 0.89 (0.02) & 0.93 (0.02) \\
    \hline    
  \end{tabular}
  \label{table:chain}
\end{table}

\begin{figure}[t]
  \centering
\subfigure{
    \includegraphics[width=0.45\columnwidth]{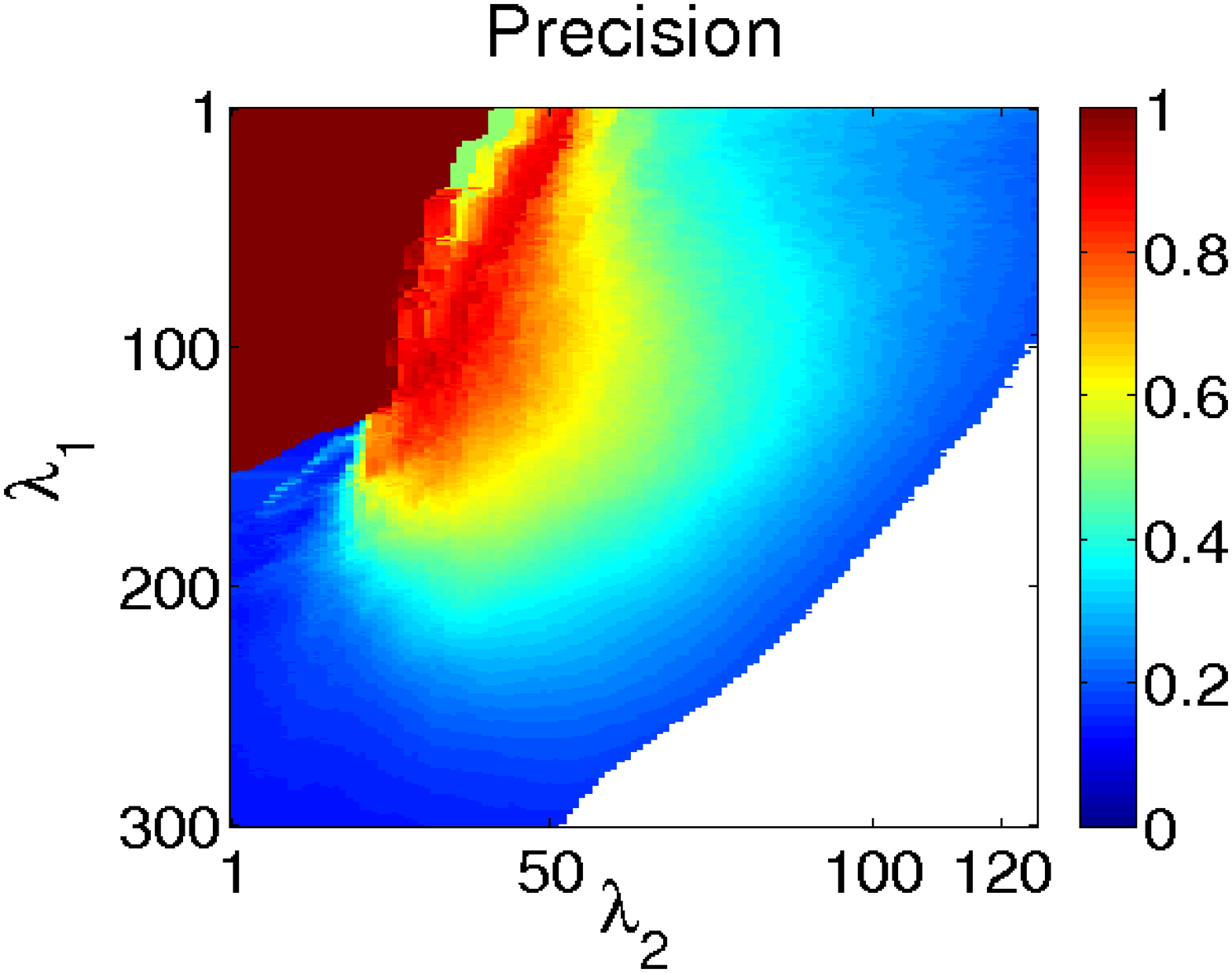}
    } 
\subfigure{
    \includegraphics[width=0.45\columnwidth]{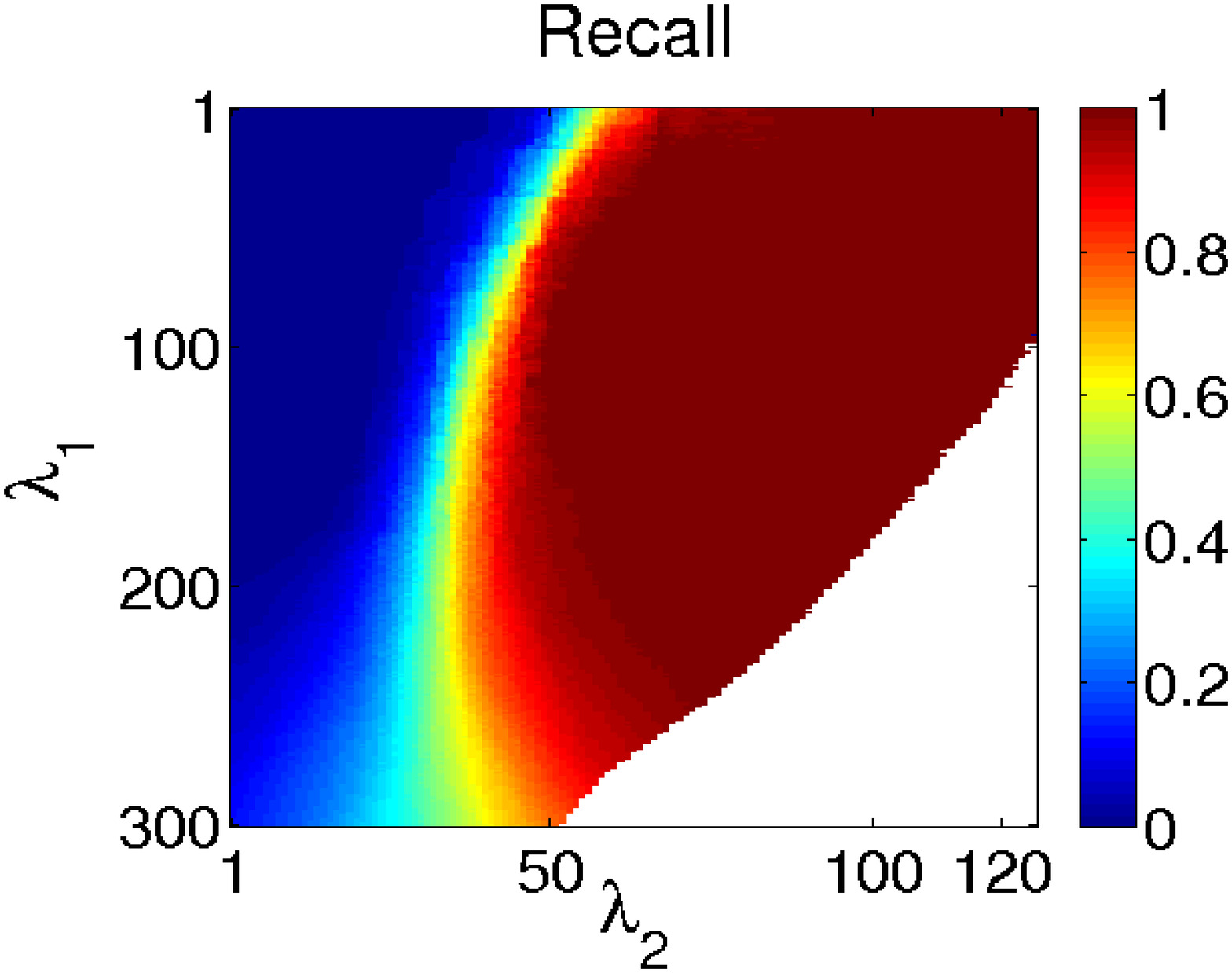}
    } 
\subfigure{
    \includegraphics[width=0.45\columnwidth]{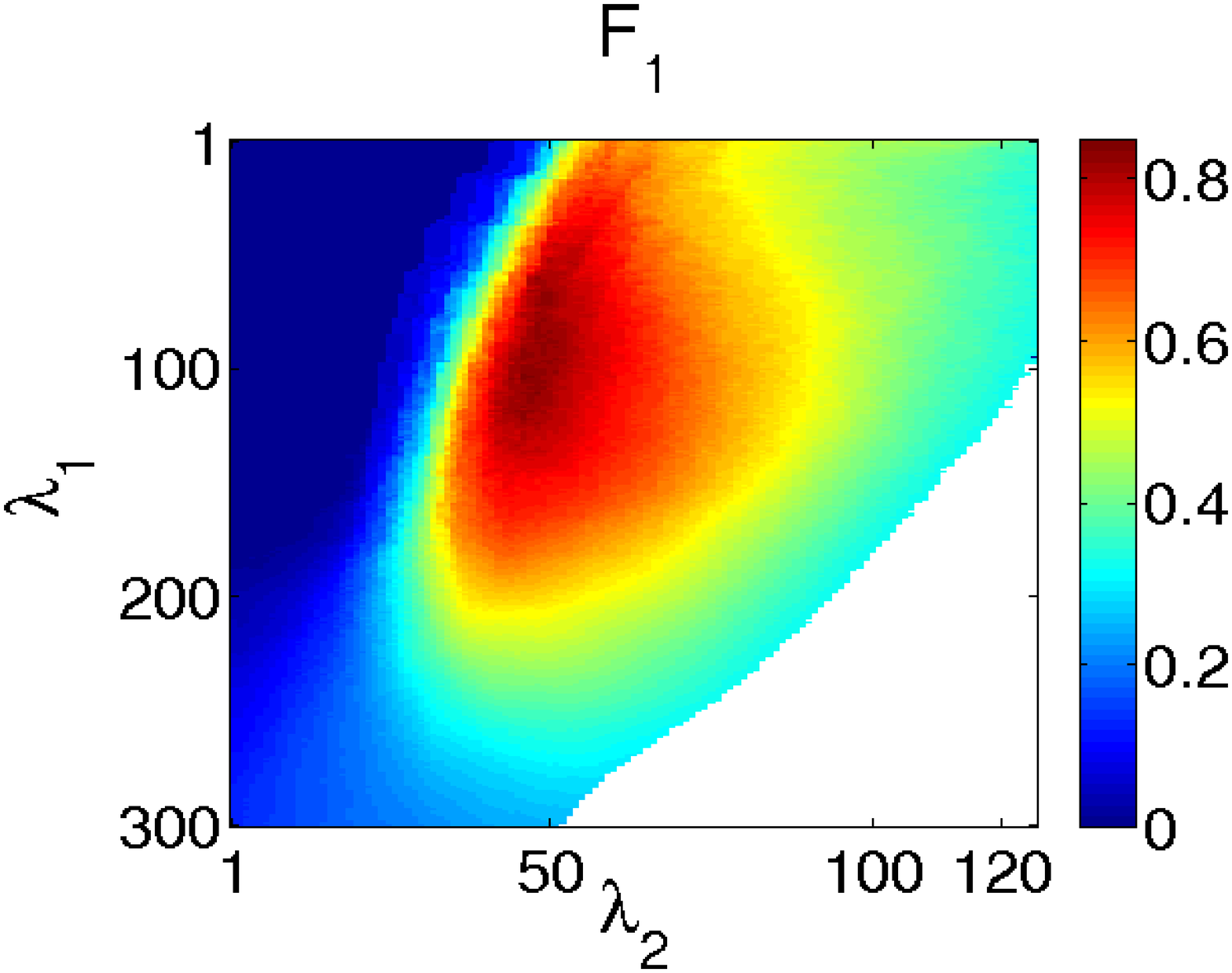}
    } 
    \caption{Plots of the precision, recall and $F_1$ scores as
      functions of the penalty parameters $\lambda_1$ and $\lambda_2$
      for chain networks. The parameter $\lambda_1$ is obtained as
      $100 * 0.98^{50 + i}$, where $i$ indexes $y$-axis. The parameter
      $\lambda_2$ is computed as $285 * 0.98^{230+j}$, where $j$
      indexes $x$-axis. The white region of each plot corresponds to a
      region of the parameter space that we did not explore.}
    \label{fig:chain}
\end{figure}

{\bf Nearest neighbors networks.} We generate nearest neighbor
networks following the procedure outlined in \cite{li06gradient}. For
each node, we draw a point uniformly at random on a unit square and
compute the pairwise distances between nodes. Each node is then
connected to 4 closest neighbors (see
Figure~\ref{fig:nn-graph}). Since some of nodes will have more than 4
adjacent edges, we remove randomly edges from nodes that have degree
larger than 4 until the maximum degree of a node in a network is
4. Each edge $(a,b)$ in this network corresponds to a non-zero element
in the precision matrix $\Omegab$, whose value is generated uniformly
on $[-1, -0.5] \cup [0.5, 1]$. The diagonal elements of the precision
matrix are set to a smallest positive number that makes the matrix
positive definite. Next, we scale the corresponding covariance matrix
$\Sigmab =\Omegab^{-1}$ to have diagonal elements equal to 1.  This
processes is repeated three times to obtain three different covariance
matrices, from which we sample $80$, $130$ and $90$ samples
respectively.

\begin{figure}[t]
  \centering
  \includegraphics[width=0.7\textwidth]{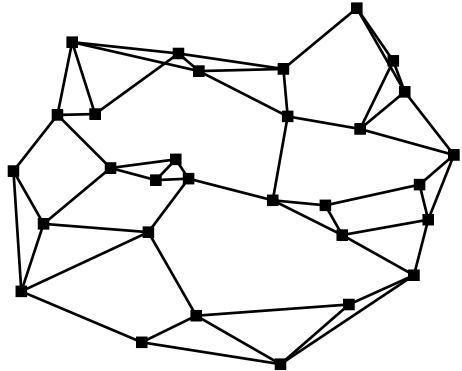}
  \caption{An instance of a random neighborhood graph with 30 nodes.}
  \label{fig:nn-graph}
\end{figure}

\begin{table}[b]
  \caption{Performance on nearest neighbor networks}
  \begin{tabular}{c||c|c|c}
    Method name & Precision & Recall & $F_1$ score \\
    \hline
    TD-Lasso & 0.79 (0.06) & 0.76 (0.05) & 0.77 (0.05) \\
    Oracle procedure & 0.87 (0.05) & 0.82 (0.05) & 0.84 (0.04) \\
    \hline    
  \end{tabular}
  \label{table:nn}
\end{table}

For illustrative purposes, Figure~\ref{fig:nn} plots the precision,
recall and $F_1$ score computed for different values of the penalty
parameters $\lambda_1$ and $\lambda_2$. Table~\ref{table:nn} shows
the precision, recall and $F_1$ score for the parameters chosen using
the BIC score, together with their standard deviations. In the same
table, we give the results of the oracle procedure.

\begin{figure}[t]
  \centering
\subfigure{
    \includegraphics[width=0.45\columnwidth]{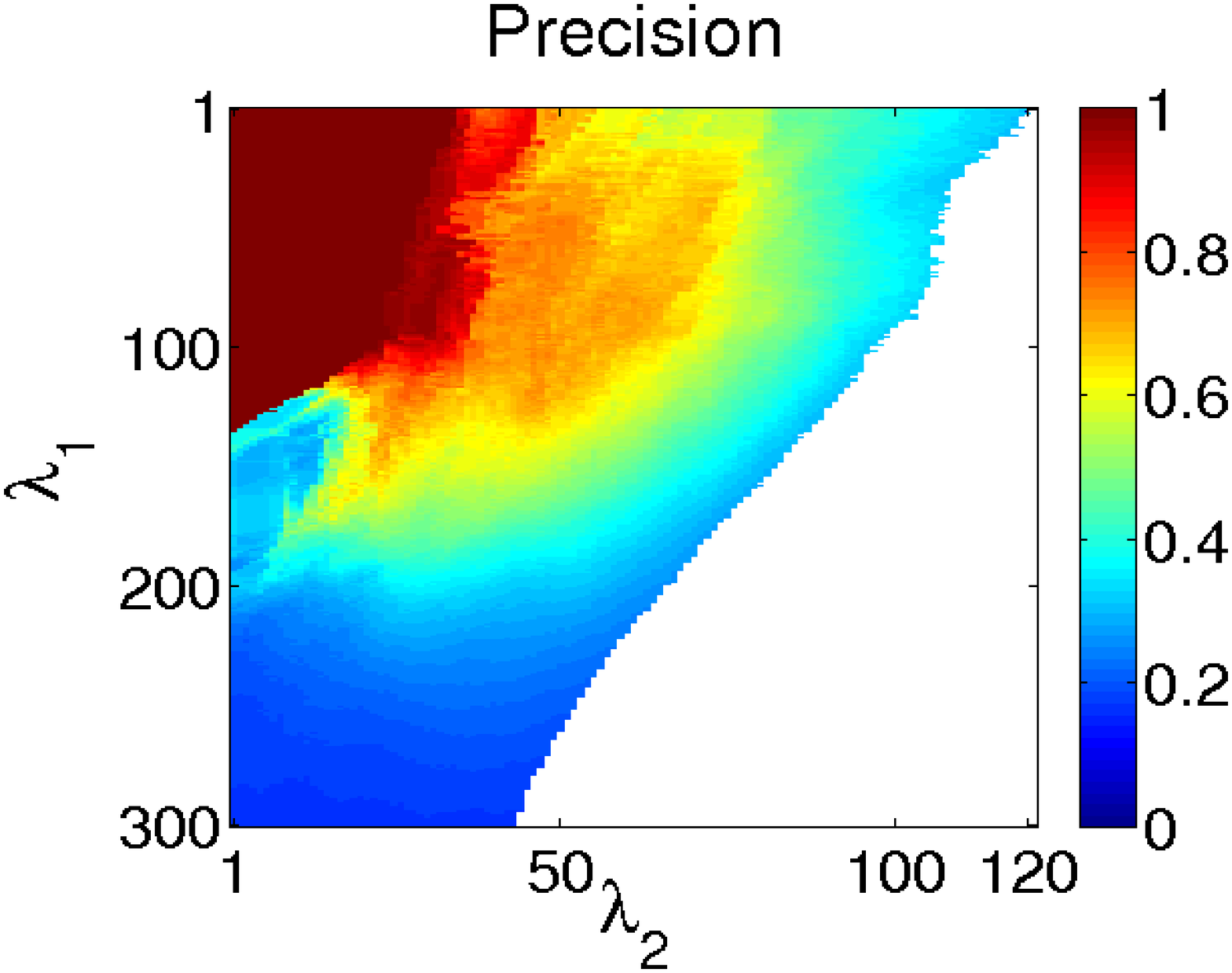}
    } 
\subfigure{
    \includegraphics[width=0.45\columnwidth]{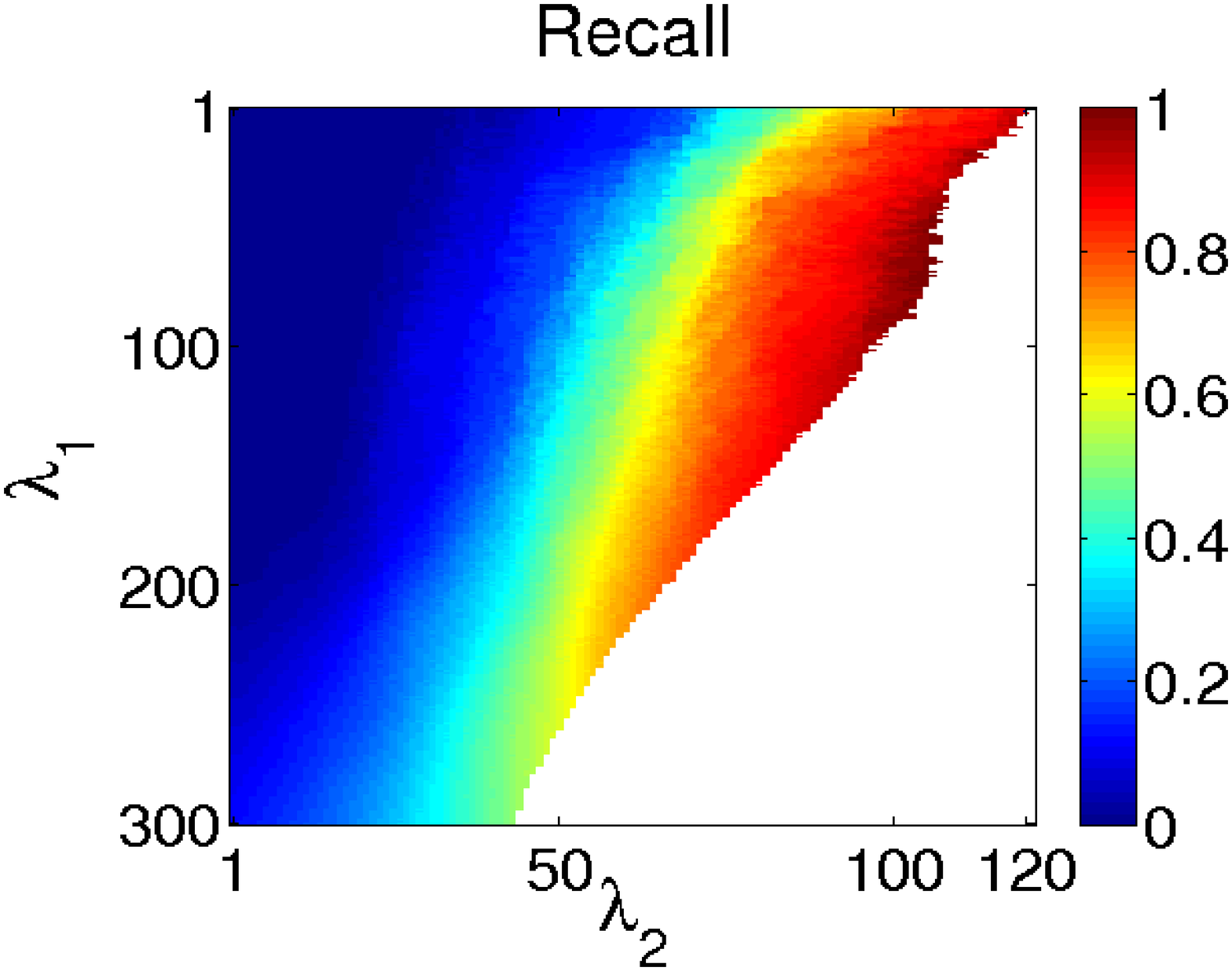}
    } 
\subfigure{
    \includegraphics[width=0.45\columnwidth]{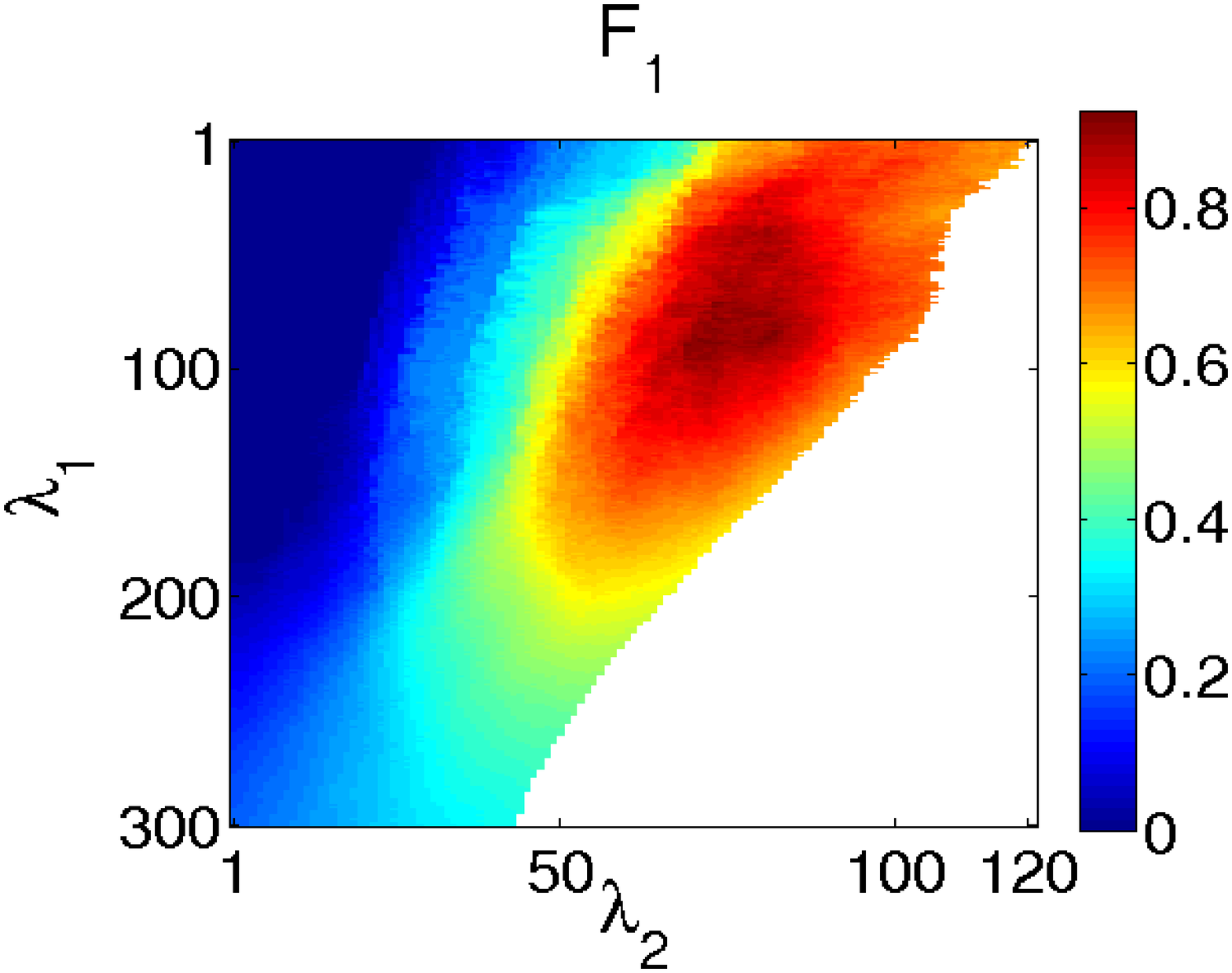}
    } 
    \caption{Plots of the precision, recall and $F_1$ scores as
      functions of the penalty parameters $\lambda_1$ and $\lambda_2$
      for nearest neighbor networks. The parameter $\lambda_1$ is
      obtained as $100 * 0.98^{50 + i}$, where $i$ indexes
      $y$-axis. The parameter $\lambda_2$ is computed as $285 *
      0.98^{230+j}$, where $j$ indexes $x$-axis. The white region of
      each plot corresponds to a region of the parameter space that we
      did not explore.}
    \label{fig:nn}
\end{figure}

\section{Conclusion} \label{sec:conclusion}

We have addressed the problem of time-varying covariance selection
when the underlying probability distribution changes abruptly at some
unknown points in time. Using a penalized neighborhood selection
approach with the fused-type penalty, we are able to consistently
estimate times when the distribution changes and the network structure
underlying the sample. The proof technique used to establish the
convergence of the boundary fractions using the fused-type penalty is
novel and constitutes an important contribution of the
paper. Furthermore, our procedure estimates the network structure
consistently whenever there is a large overlap between the estimated
blocks and the unknown true blocks of samples coming from the same
distribution. The proof technique used to establish the consistency of
the network structure builds on the proof for consistency of the
neighborhood selection procedure, however, important modifications are
necessary since the times of distribution changes are not known in
advance. Applications of the proposed approach range from cognitive
neuroscience, where the problem is to identify changing associations
between different parts of a brain when presented with different
stimuli, to system biology studies, where the task is to identify
changing patterns of interactions between genes involved in different
cellular processes.  We conjecture that our estimation procedure is
also valid in the high-dimensional setting when the number of
variables $p$ is much larger than the sample size $n$. We leave the
investigations of the rate of convergence in the high-dimensional
setting for a future work.

\section{Proofs} \label{sec:proofs}

\subsection{Proof of Lemma~\ref{lem:kkt_conditions}}

For each $i \in [n]$, introduce a $(p-1)$-dimensional vector
$\gammab_i$ defined as
\begin{equation*}
  \gammab_i = \left\{
  \begin{array}{ll}
    \betab_{\cdot, i} & \text{for } i = 1\\
    \betab_{\cdot, i} - \betab_{\cdot, i-1} & \text{otherwise}
  \end{array}
  \right.
\end{equation*}
and rewrite the objective
\eqref{eq:penalized-estimation} as 
\begin{equation}
  \label{eq:change_points:rewritten:high_dim}
  \begin{aligned}
  \{ \hat \gammab^i \}_{i \in [n]} =
  \argmin_{\gammab \in \RR^{n \times p-1}}\ \sum_{i=1}^n (x_{i,a} &-
  \sum_{b \in \bks a} x_{i,b} \sum_{j
    \leq i} \gamma_{j,b})^2 \\ &+ 2\lambda_1 \sum_{i=2}^n
  \norm{\gammab_i}_2 + 2\lambda_2 
   \sum_{i=1}^n \sum_{b \in \bks a} |\sum_{j \leq i}\gamma_{j,b}|.
 \end{aligned}
\end{equation}
A necessary and sufficient condition for $\{ \hat \gammab_i \}_{i
  \in [n]}$ to be a solution of
\eqref{eq:change_points:rewritten:high_dim}, is that for each $k \in
[n]$ the $(p-1)$-dimensional zero vector, $\0b$, belongs to the
subdifferential of \eqref{eq:change_points:rewritten:high_dim} with
respect to $\gammab_k$ evaluated at $\{ \hat \gammab_i \}_{i \in
  [n]}$, that is,
\begin{equation}
  \label{eq:partial_derivative}
    \0b  = 2\sum_{i=k}^n (-\xb\iba) (x_{i,a} -
    \sum_{b \in \bks a} x_{i,b} \hat\beta_{b,i}^a) + 2\lambda_1 \hat 
    \zb_k + 2 \lambda_2 \sum_{i=k}^n \hat\yb_i,
\end{equation}
where  $\hat \zb_k \in \partial \norm{\cdot}_2(\hat \gammab_k)$, that is,
\begin{equation*}
  \label{eq:subgradient:high_dim}
  \tilde \zb_k = \left\{ 
    \begin{array}{ll}
      \frac{\tilde \gammab_k}{\norm{\tilde \gammab_k}_2} & \text{if }
      \tilde \gammab_k \neq 0\\
      \in \Bcal_2(0, 1) & \text{otherwise}
    \end{array}
   \right.
 \end{equation*}
 and for $k \leq i$, $\hat \yb_i \in \partial |\sum_{j \leq i} \hat
 \gammab_j|$, that is, $\yb_i = \sgn(\sum_{j \leq i} \hat \gammab_j)$
 with $\sgn(0) \in [-1, 1]$.  The Lemma now simply follows from
 \eqref{eq:partial_derivative}.

\subsection{Proof of
  Theorem~\ref{thm:consistent-estimation-boundaries:known-num-blocks}}

We build on the ideas presented in the proof of Proposition 5 in
\cite{harchaoui10multiple}. Using the union bound,
\begin{equation*}
  \PP[\max_{j \in [B]} |T_j - \hat T_j| > n\delta_n] 
    \leq \sum_{j \in [B]} \PP[|T_j - \hat T_j| > n\delta_n]
\end{equation*}
and it is enough to show that  $\PP[|T_j - \tilde T_j| > n\delta_n]
\rightarrow 0$ for all $j \in [B]$. Define the event $A_{n,j}$ as 
\begin{equation*}
  A_{n,j} := \big\{ |T_j - \hat T_j| > n\delta_n \big\}  
\end{equation*}
and the event $C_n$ as 
\begin{equation*}
  C_n := \bigg\{ \max_{j \in [B]} |\hat T_j - T_j| < \frac{\Delta_{\min}}{2} \bigg\}.
\end{equation*}
We show that $\PP[A_{n,j}] \rightarrow 0$ by showing that both
$\PP[A_{n,j} \cap C_n] \rightarrow 0$ and $\PP[A_{n,j} \cap C_n^c]
\rightarrow 0$ as $n \rightarrow \infty$. The idea here is that, in
some sense, the event $C_n$ is a good event on which the estimated
boundary partitions and the true boundary partitions are not too far
from each other. Considering the two cases will make the analysis
simpler.

First, we show that $\PP[A_{n,j} \cap C_n] \rightarrow 0$. Without
loss of generality, we assume that $\hat T_j < T_j$, since the other
case follows using the same reasoning. Using \eqref{eq:kkt_conditions}
twice with $k = \hat T_j$ and with $k = T_j$ and then applying the
triangle inequality we have
\begin{equation} \label{thm:1:kkt-bound}
    2\lambda_1 \geq 
    \bignorm{
      \sum_{i=\hat T_j}^{T_{j}-1} \xb\iba 
       \dotp{\xb\iba}{\hat \betab_{\cdot, i} - \betab_{\cdot, i}}
    - \sum_{i=\hat T_j}^{\hat T_j - 1} \xb\iba \epsilon_i 
    + \lambda_2 \sum_{i=\hat T_{j}}^{T_{j}-1} \hat \yb_i}_2.
\end{equation}
Some algebra on the above display gives
\begin{equation*}
  \label{eq:event:c_n,k}
  \begin{aligned}
     2\lambda_1 + (T_j - \hat T_j)\sqrt{p} \lambda_2
      &\geq \bignorm{
        \sum_{i=\hat T_j}^{T_j - 1} \xb\iba
        \langle\xb\iba, \thetab^{j} - \thetab^{j+1}\rangle
      }_2 \\
      &\quad -
        \bignorm{ \sum_{i = \hat T_j}^{T_j - 1} \xb\iba
          \langle \xb\iba, \thetab^{j+1} - \hat \thetab^{j+1} \rangle
      }_2 -
      \bignorm{
        \sum_{i=\hat T_j}^{T_j - 1} \xb\iba \epsilon_i
      }_2 \\ 
      & =: \norm{R_1}_2 - \norm{R_2}_2 - \norm{R_3}_2.
    \end{aligned}
\end{equation*}
The above display occurs with probability one, so that the event $\{
2\lambda_1 + (T_j - \hat T_j)\sqrt{p}\lambda_2 \geq \frac{1}{3} \norm{
  R_1 }_2 \} \cup \{ \norm{R_2}_2 \geq \frac{1}{3} \norm{R_1}_2 \}
\cup \{ \norm{R_3}_2 \geq \frac{1}{3} \norm{R_1}_2 \}$ also occurs
with probability one, which gives us the following bound
\begin{equation*}
  \begin{aligned}
    \PP[A_{n,j} \cap C_n] 
    & \leq \PP[A_{n,j} \cap C_n \cap 
               \{ 2\lambda_1 +  (T_j - \hat T_j)\sqrt{p} \lambda_2
                \geq \frac{1}{3} \norm{ R_1 }_2 \}] \\
    & \quad + \PP[A_{n,j} \cap C_n \cap 
               \{ \norm{R_2}_2 \geq \frac{1}{3} \norm{R_1}_2 \}] \\
    & \quad + \PP[A_{n,j} \cap C_n \cap 
               \{ \norm{R_3}_2 \geq \frac{1}{3} \norm{R_1}_2 \}] \\
    & =: \PP[A_{n,j,1}] + \PP[A_{n,j,2}] +  \PP[A_{n,j,3}].
  \end{aligned}
\end{equation*}
First, we focus on the event $A_{n,j,1}$. Using
lemma~\ref{lem:concentration-eigenvalue:supremum-over-blocks}, we can
upper bound $\PP[A_{n, j, 1}]$ with
\begin{equation*}
  \PP[2\lambda_1 + (T_j - \hat T_j)\sqrt{p}\lambda_2 \geq
  \frac{\phi_{\min}}{27} (T_j - \hat T_j) \xi_{\min}]  
  + 2\exp(-n\delta_n/2 + 2 \log n).
\end{equation*}
Since under the assumptions of the theorem
$(n\delta_n\xi_{\min})^{-1}\lambda_1 \rightarrow 0$ and
$\xi_{\min}^{-1}\sqrt{p}\lambda_2 \rightarrow 0$ as $n \rightarrow
\infty$, we have that $\PP[A_{n,j,1}] \rightarrow 0$ as $n \rightarrow
\infty$.


Next, we show that the probability of the event $A_{n,j,2}$ converges
to zero. Let $\bar T_j := \lfloor2^{-1}(T_j + T_{j+1})
\rfloor$. Observe that on the event $C_n$, $\hat T_{j+1} > \bar T_j$
so that $\hat \betab_{\cdot,i} = \hat \thetab^{j+1}$ for all $i \in
[T_j, \bar T_j]$.  Using~\eqref{eq:kkt_conditions} with $k = T_j$ and
$k = \bar T_j$ we have that
\begin{equation*}
  2\lambda_1 + (\bar T_j - T_j) \sqrt{p}\lambda_2 
  \geq
  \bignorm{ \sum_{i= T_j}^{\bar T_j-1}
    \xb\iba \dotp{\xb\iba}{\thetab^{j+1} - \hat \thetab^{j+1}}
  }_2 - 
  \bignorm{\sum_{i= T_j}^{\bar T_j-1} \xb\iba \epsilon_i}_2.
\end{equation*}
Using lemma~\ref{lem:concentration-eigenvalue:supremum-over-blocks} on
the display above we have
\begin{equation}
  \label{eq:bound:anj2}  
  \norm{\thetab^{j+1} - \hat \thetab^{j+1}}_2 \leq
  \frac{36\lambda_1 + 18(\bar T_j - T_j) \sqrt{p}\lambda_2 +
    18\bignorm{\sum_{i = T_j}^{\bar T_j-1} \xb\iba \epsilon_i}_2}
  {(T_{j+1} - T_j)\phi_{\min}},
\end{equation}
which holds with probability at least $1 - 2\exp(-\Delta_{\min}/4 + 2 \log
n)$. We will use the above bound to deal with the event $\{
\norm{R_2}_2 \geq \frac{1}{3} \norm{R_1}_2 \}$. Using
lemma~\ref{lem:concentration-eigenvalue:supremum-over-blocks}, we have
that $\phi_{\min}(T_j - \hat T_j)\xi_{\min}/9 \leq \norm{R_1}_2$ and
$\norm{R_2}_2 \leq (T_j - \hat T_j)9\phi_{\max} \norm{\thetab^{j+1}
  - \hat \thetab^{j+1}}_2$ with probability at least $1 -
4\exp(-n\delta_n/2 + 2 \log n)$. Combining with~\eqref{eq:bound:anj2},
the probability $\PP[A_{n,j,2}]$ is upper bounded by
\begin{equation*}
\begin{aligned}
  \PP&[c_1\phi_{\min}^2\phi_{\max}^{-1}\Delta_{\min}\xi_{\min} 
       \leq \lambda_1] + 
  \PP[c_2\phi_{\min}^2\phi_{\max}^{-1}\xi_{\min} 
       \leq \sqrt{p} \lambda_2] + \\
  & \PP[c_3\phi_{\min}^2\phi_{\max}^{-1}\xi_{\min} 
       \leq (\bar T_j - T_j)^{-1}\norm{\sum_{i = T_j}^{\bar T_j-1} \xb\iba
           \epsilon_i}_2] +
       c_4\exp(-n\delta_n/2 + 2 \log n).
\end{aligned}
\end{equation*}
Under the conditions of the theorem, the first term above converges to
zero, since $\Delta_{\min} > n\delta_n$ and $(n\delta_n
\xi_{\min})^{-1} \lambda_1 \rightarrow 0$.  The second term also
converges to zero, since $\xi_{\min}^{-1}\sqrt{p}\lambda_2 \rightarrow
0$. Using lemma~\ref{lem:concentration-error-term}, the third term
converges to zero with the rate $\exp(-c_6 \log n)$,
since $(\xi_{\min} \sqrt{ \Delta_{\min} })^{-1} \sqrt{p\log n}
\rightarrow 0$. Combining all the bounds, we have that $\PP[A_{n,j,2}]
\rightarrow 0$ as $n \rightarrow \infty$.


Finally, we upper bound the probability of the event $A_{n,j,3}$.  As
before, $\phi_{\min}(T_j - \hat T_j)\xi_{\min}/9 \leq \norm{R_1}_2$
with probability at least $1 - 2\exp(-n\delta_n/2 + 2\log n)$. This
gives us an upper bound on $\PP[A_{n,j,3}]$ as
\begin{equation*}
  \PP\bigg[\frac{\phi_{\min} \xi_{\min}}{27} \leq 
    \frac{\norm{\sum_{i = \hat T_j}^{T_j-1} \xb\iba \epsilon_i}_2}
         {T_j - \hat T_j}
    \bigg] + 2\exp(-n\delta_n/2 + 2\log n),
\end{equation*}
which, using lemma~\ref{lem:concentration-error-term}, converges to
zero as under the conditions of the theorem $(\xi_{\min}
\sqrt{n\delta_n})^{-1} \sqrt{p\log n} \rightarrow 0$. Thus we have
shown that $\PP[A_{n,j,3}] \rightarrow 0$. Since the case when $\hat
T_j > T_j$ is shown similarly, we have proved that $\PP[A_{n,j} \cap
C_n] \rightarrow 0$ as $n \rightarrow \infty$.


We proceed to show that $\PP[A_{n,j} \cap C_n^c] \rightarrow 0$
as $n \rightarrow \infty$. Recall that $ C_n^c = \{ \max_{j \in [B]}
|\hat T_j - T_j| \geq \Delta_{\min} / 2 \}$. Define the following
events
\begin{align*}
  D_n^{(l)} &:= 
     \left\{ \exists j \in [B],\ \hat T_j \leq T_{j-1} \right\} 
     \cap  C_n^c, \\
  D_n^{(m)} &:= 
     \left\{ \forall j \in [B],\ T_{j-1} < \hat T_j < T_{j+1} \right\} 
     \cap  C_n^c, \\
  D_n^{(r)} &:= 
     \left\{ \exists j \in [B],\ \hat T_j \geq T_{j+1} \right\} 
     \cap  C_n^c \\
\end{align*}
and write $\PP[A_{n,j} \cap C_n^c] = \PP[A_{n,j} \cap D_n^{(l)}] +
\PP[A_{n,j} \cap D_n^{(m)}] + \PP[A_{n,j} \cap D_n^{(r)}]$. First,
consider the event $A_{n,j} \cap D_n^{(m)}$ under the assumption that
$\hat T_j \leq T_j$. Due to symmetry, the other case will follow in a
similar way. Observe that
\begin{equation}
  \label{eq:step2:middle}
  \begin{aligned}
    &\PP[A_{n,j} \cap D_n^{(m)}] \\
      & \leq \PP[A_{n,j} \cap 
          \{ (\hat T_{j+1} - T_j) \geq \frac{\Delta_{\min}}{2} \} \cap
          D_n^{(m)}] \\
      & \quad +
      \PP[\{ (T_{j+1} - \hat T_{j+1}) \geq \frac{\Delta_{\min}}{2} \} \cap
          D_n^{(m)}] \\
      & \leq \PP[A_{n,j} \cap 
          \{ (\hat T_{j+1} - T_j) \geq \frac{\Delta_{\min}}{2} \} \cap
          D_n^{(m)}] \\
      & \quad + \sum_{k = j + 1}^{B - 1}       
         \PP[\{ (T_{k} - \hat T_{k}) \geq \frac{\Delta_{\min}}{2} \} \cap
             \{ (\hat T_{k+1} - T_{k}) \geq \frac{\Delta_{\min}}{2} \}
         \cap  D_n^{(m)}]. 
  \end{aligned}
\end{equation}
We bound the first term in \eqref{eq:step2:middle} and note that the
other terms can be bounded in the same way. The following analysis is
performed on the event $A_{n,j} \cap \{ (\hat T_{j+1} - T_j) \geq
\Delta_{\min}/2 \} \cap D_n^{(m)}$.  Using
\eqref{eq:kkt_conditions} with $k = \hat T_j$ and $k = T_j$, after
some algebra (similar to the derivation of~\eqref{thm:1:kkt-bound}) the
following holds
\begin{equation*}
  \norm{\thetab^{j} - \hat \thetab^{j+1}}_2 \leq
  \frac{18\lambda_1 + 9(T_j - \hat T_j)\sqrt{p}\lambda_2 + 
    9\norm{\sum_{i =\hat T_j}^{T_j - 1} \xb\iba \epsilon_i}} 
    {\phi_{\min}(T_j - \hat T_j)},
\end{equation*}
with probability at least $1 - 2\exp(-n\delta_n/2 + 2 \log n)$, where
we have used
lemma~\ref{lem:concentration-eigenvalue:supremum-over-blocks}. Let
$\bar T_j = \lfloor 2^{-1}(T_j + T_{j+1})\rfloor$.
Using~\eqref{eq:kkt_conditions} with $k = \bar T_{j}$ and $k = T_j$
after some algebra (similar to the derivation
of~\eqref{eq:bound:anj2}) we obtain the following bound
\begin{equation*}
\begin{aligned}
  \norm{\thetab^{j} - \thetab^{j+1}}_2 \leq &
   \frac{18\lambda_1 + 9(\bar T_j - T_j)\sqrt{p}\lambda_2 +
     9\norm{\sum_{i = T_j}^{\bar T_j - 1} \xb\iba \epsilon_i}_2}
    { \phi_{\min}(\bar T_j - T_j) } \\ &+
   81 \phi_{\max}\phi_{\min}^{-1} 
     \norm{\thetab^{j} - \hat \thetab^{j+1}}_2,
\end{aligned}
\end{equation*}
which holds with probability at least $1 - c_1\exp(- n\delta_n / 2 + 2
\log n)$, where we have used lemma
\ref{lem:concentration-eigenvalue:supremum-over-blocks} twice.
Combining the last two displays, we can upper bound the first term in
\eqref{eq:step2:middle} with
\begin{equation*}
\begin{aligned}
  \PP&[\xi_{\min}n\delta_n \leq c_1 \lambda_1] +
  \PP[\xi_{\min} \leq c_2 \sqrt{p}\lambda_2] \\
  & + \PP[\xi_{\min}\sqrt{n\delta_n} \leq c_3 \sqrt{p\log n}]
  + c_4\exp(- c_5 \log n),
\end{aligned}
\end{equation*}
where we have used lemma~\ref{lem:concentration-error-term} to obtain
the third term. Under the conditions of the theorem, all terms
converge to zero. Reasoning similar about the other terms in
\eqref{eq:step2:middle}, we can conclude that $\PP[A_{n,j} \cap
D_n^{(m)}] \rightarrow 0$ as $n \rightarrow \infty$.

Next, we bound the probability of the event $A_{n,j} \cap D_n^{(l)}$,
which is upper bounded by 
\begin{equation*}
  \PP[D_n^{(l)}] \leq \sum_{j = 1}^{B}
  2^{j-1} \PP[\max \{l \in [B] : \hat T_l \leq T_{l-1}\} = j].
\end{equation*}
Observe that 
\begin{equation*}  
  \begin{aligned}
    \{\max & \{l \in [B] : \hat T_l \leq T_{l-1}\} = j \} \\
       & \subseteq \bigcup_{l = j}^B
             \{ T_j - \hat T_j \geq \frac{\Delta_{\min}}{2}\}    \cap 
           \{\hat T_{j+1} - T_j \geq \frac{\Delta_{\min}}{2}\}
           \\
  \end{aligned}
\end{equation*}
so that we have
\begin{equation*}
  \label{eq:step2:bound:left_term}
  \begin{aligned}
    \PP[D_n^{(l)}] & 
     \leq 2^{B - 1} \sum_{j=1}^{B-1} \sum_{l > j} 
     \PP[\{ T_l - \hat T_l \geq \frac{\Delta_{\min}}{2}\} 
         \cap 
         \{\hat T_{l+1} - T_l \geq \frac{\Delta_{\min}}{2}\}].
  \end{aligned}
\end{equation*}
Using the same arguments as those used to bound terms in
\eqref{eq:step2:middle}, we have that $\PP[D_n^{(l)}] \rightarrow 0$
as $n \rightarrow \infty$ under the conditions of the theorem.
Similarly, we can show that the term $\PP[D_n^{(r)}] \rightarrow 0$ as
$n \rightarrow \infty$. Thus, we have shown that $\PP[A_{n,j} \cap
C_n^c] \rightarrow 0$, which concludes the proof.

\subsection{Proof of Lemma~\ref{lem:unique_pattern}}

Consider $\hat \Tcal$ fixed. The lemma is a simple consequence of the
duality theory, which states that given the subdifferential $\hat
\yb_i$ (which is constant for all $i \in \hat \Bcal^j$, $\hat \Bcal^j$
being an estimated block of the partition $\hat \Tcal$), all solutions $\{\check
\betab_{\cdot, i}\}_{i \in [n]}$ of \eqref{eq:penalized-estimation}
need to satisfy the complementary slackness condition $\sum_{b \in
  \bks a} \hat y_{i,b} \check \beta_{b, i} = \norm{\check \betab_{\cdot,
    i}}_1$, which holds only if $\check \beta_{b, i} = 0$ for all $b
\in \bks a$ for which $| \hat y_{i,b} | < 1$. 

\subsection{Proof of Theorem~\ref{thm:correct-neighborhood}}

Since the assumptions of theorem
\ref{thm:consistent-estimation-boundaries:known-num-blocks} are
satisfied, we are going to work on the event 
\begin{equation*}
  \Ecal := \{ \max_{j \in [B]}
  |\hat T_j - T_j| \leq n\delta_n \}.  
\end{equation*}
In this case, $|\hat \Bcal^k| = \Ocal(n)$. For $i \in \hat \Bcal^k$,
we write
\begin{equation}
\label{eq:model:proof-neighborhood}
  x_{i,a} = \sum_{b \in S^j} x_{i,b} \theta_b^k +  e_i + \epsilon_i
\end{equation}
where $e_i = \sum_{b \in S} x_{i,b}(\beta_{b, i} - \theta_b^k)$
is the bias. Observe that $\forall i \in \hat \Bcal^k \cap \Bcal^k$,
the bias $ e_i = 0$, while for $i \not\in \hat \Bcal^k \cap
\Bcal^k$, the bias $e_i$ is normally distributed with variance bounded
by $M^2\phi_{\max}$ under the assumption {\bf A1} and {\bf A3}.

We proceed to show that $S(\hat \thetab^k) \subset S^k$. Since $\hat
\thetab^k$ is an optimal solution of~\eqref{eq:penalized-estimation},
it needs to satisfy
\begin{equation}
  \label{eq:kkt-known-boundaries}
\begin{aligned}
  (\Xb\ba\supb{k})'\Xb\ba\supb{k}  (\hat \thetab^k - 
  \thetab^k) &- (\Xb\ba\supb{k})'( \eb\supb{k} + \epsilonb\supb{k})
  \\ &
  + \lambda_1(\hat \zb_{\hat T_{k-1}} - \hat \zb_{\hat T_k}) + 
  \lambda_2 |\hat \Bcal^k| \hat \yb_{\hat T_{k-1}} = 0.
\end{aligned}
\end{equation}
Now, we will construct the vectors $\check \thetab^k, \check \zb_{\hat
  T_{k-1}}, \check \zb_{\hat T_k}$ and $\check \yb_{\hat T_{k-1}}$
that satisfy \eqref{eq:kkt-known-boundaries} and verify that the
subdifferential vectors are dual feasible. Consider the following
restricted optimization problem
\begin{equation}
  \label{eq:objective_known_boundaries:restricted}
\begin{aligned}
  \min_{\thetab^{1}, \ldots, \thetab^{\hat B};
    \ \thetab_{N^k}^k = \0b }\
  & \sum_{j \in [\hat B]} 
     \norm{\Xb_a\supb{j} - \Xb\ba\supb{j}\thetab^{j}}_2^2 \\
  & + 2\lambda_1 
      \sum_{j = 2}^{\hat B} \norm{\thetab^{j} - \thetab^{j-1}}_2
  + 2\lambda_2 \sum_{j=1}^{\hat B} |\hat \Bcal^j|\norm{\thetab^j}_1,
\end{aligned}
\end{equation}
where the vector $\thetab_{N^k}^k$ is constrained to be $\0b$. Let
$\{\check \thetab^j\}_{j \in [\hat B]}$ be a solution to the
restricted optimization problem
\eqref{eq:objective_known_boundaries:restricted}. Set the subgradient
vectors as $\check \zb_{\hat T_{k-1}} \in \partial \norm{\check
  \thetab^k - \check \thetab^{k-1}}$, $\check \zb_{T_k} \in \partial
\norm{\check \thetab^{k+1} - \check \thetab^{k}}$ and $\check
\yb_{\hat T_{k-1}, S^k} = \sgn(\check \thetab_{S^k}^k)$. Solve
\eqref{eq:kkt-known-boundaries} for $\check \yb_{\hat T_{k-1},N^k}$. By
construction, the vectors $\check \thetab^k, \check \zb_{\hat
  T_{k-1}}, \check \zb_{\hat T_k}$ and $\check \yb_{\hat T_{k-1}}$
satisfy \eqref{eq:kkt-known-boundaries}. Furthermore, the vectors
$\check \zb_{\hat T_{k-1}}$ and $\check \zb_{\hat T_k}$ are elements
of the subdifferential, and hence dual feasible. To show that $\check
\thetab^k$ is also a solution to
\eqref{eq:objective_known_boundaries}, we need to show that
$\norm{\check \yb_{\hat T_{k-1}, N^k}}_{\infty} \leq 1$, that is, that
$\check \yb^{\hat T_{k-1}}$ is also dual feasible variable. Using
lemma~\ref{lem:unique_pattern}, if we show that $\check \yb_{\hat
  T_{k-1}, N^k}$ is strict dual feasible, $\norm{\check \yb_{\hat
    T_{k-1},N^k}}_{\infty} < 1$, then any other solution $\hat{\check
  \thetab}^k$ to \eqref{eq:objective_known_boundaries} will satisfy
$\hat{\check \thetab}_N^k = \0b$.

From~\eqref{eq:kkt-known-boundaries} we can obtain an explicit formula
for $\check \thetab_{S^k}$
\begin{equation}
  \label{eq:explict-solution}
  \begin{aligned}
  \check \thetab_{S^k}^k =  \thetab_{S^k}^k &+ 
    \rbr{(\Xb_{S^k}\supb{k})'\Xb\sk\supb{k}}^{-1} 
    (\Xb\sk\supb{})'(\eb\supb{k} + \epsilonb\supb{k}) \\&-
     \rbr{(\Xb\sk\supb{k})'\Xb\sk\supb{k}}^{-1} \rbr{
    \lambda_1(\check \zb_{\hat T_{k-1},S^k} - \check \zb_{\hat T_k, S^k}) +
    \lambda_2 |\hat \Bcal^k| \check \yb_{\hat T_{k-1},S^k}}.
  \end{aligned}
\end{equation}
Recall that for large enough $n$ we have that $|\hat \Bcal| > p$, so
that the matrix $(\Xb\sk\supb{k})'\Xb\sk\supb{k}$ is invertible with
probability one. Plugging \eqref{eq:explict-solution} into
\eqref{eq:kkt-known-boundaries}, we have that $\norm{\check \yb_{\hat
    T_{k-1},N^k}}_{\infty} < 1$ if $\max_{b \in N^k} |Y_b| < 1$, where
$Y_b$ is defined to be
\begin{equation}
  \label{eq:sparse_pattern_variable}
\begin{aligned}
Y_b :=  \rbr{\Xb_b\supb{k}}'\bigg[\Xb\sk\supb{k} 
       &\rbr{(\Xb\sk\supb{k})'\Xb\sk\supb{k}}^{-1}
       \Big(\check \yb_{\hat T_{k-1},S^k} + 
         \frac{\lambda_1(\hat \zb_{\hat T_{k-1},S^k} 
           - \hat \zb_{\hat T_k, S^k})}{|\hat \Bcal^k|\lambda_2}\Big)
       \\&+ \Hb\sk^{\hat \Bcal^k, \perp}\Big(
         \frac{\eb\supb{k} + \epsilonb\supb{k}}
              {|\hat \Bcal^k|\lambda_2}\Big)
      \bigg]
      - \frac{\lambda_1(\check z_{\hat T_{k-1},b} 
        - \check z_{\hat T_k, b})}{|\hat \Bcal^k|\lambda_2},
\end{aligned}
\end{equation}
where $\Hb\sk^{\hat \Bcal^k, \perp}$ is the projection matrix
\begin{equation*}
  \Hb\sk^{\hat \Bcal^k, \perp} = \Ib - 
    \Xb\sk\supb{k}\rbr{(\Xb\sk\supb{k})'\Xb\sk\supb{k}}^{-1}
    \rbr{\Xb\sk\supb{k}}'.
\end{equation*}
Let $\tilde \Sigmab^k$ and $\hat{\tilde \Sigmab}^k$ be defined as 
\begin{equation*}
  \tilde \Sigmab^k = \frac{1}{|\hat \Bcal^k|} 
  \sum_{i \in \hat \Bcal^k} 
  \EE[ \xb\ba^i (\xb\ba^i)' ] \quad \text{and} \quad
  \hat{\tilde \Sigmab}^k = \frac{1}{|\hat \Bcal^k|} 
  \sum_{i \in \hat \Bcal^k} 
   \xb\ba^i (\xb\ba^i)'.  
\end{equation*}
For $i \in [n]$, we let $\Bcal(i)$ index the block to which the sample
$i$ belongs to.  Now, for any $b \in N^k$, we can write $x_b^i =
\Sigmab_{bS^k}^{\Bcal(i)} \big(\Sigmab_{S^kS^k}^{\Bcal(i)}\big)^{-1}
\xb\sk^i + w_b^i$ where $w_{b}^i$ is normally distributed with
variance $\sigma_b^2 < 1$ and independent of $\xb\sk^i$. Let $\Fb_b
\in \mathbb{R}^{|\hat\Bcal^k|}$ be the vector whose components are
equal to $\Sigmab_{bS^k}^{\Bcal(i)}
\big(\Sigmab_{S^kS^k}^{\Bcal(i)}\big)^{-1} \xb\sk^i$, $i \in
\hat\Bcal^k$, and $\Wb_b \in \RR^{|\hat\Bcal^k|}$ be the vector with
components equal to $w_b^i$. Using this notation, we write $Y_b =
T_b^1 + T_b^2 + T_b^3 + T_b^4$ where
\begin{equation}
  \label{eq:sparse_event:T1a}
  T_b^1 = \Fb_b' \Xb\sk\supb{k}
       \rbr{(\Xb\sk\supb{k})'\Xb\sk\supb{k}}^{-1}
       \Big(\check \yb_{\hat T_{k-1}} + 
         \frac{\lambda_1(\check \zb_{\hat T_{k-1},S^k} 
           - \check \zb_{\hat T_k,S^k})}{|\hat \Bcal^k|\lambda_2}\Big)
\end{equation}
\begin{equation}
  \label{eq:sparse_event:T1b}
  T_b^2 = \Fb_b' \Hb\sk^{\hat \Bcal^k, \perp}\Big(
         \frac{\eb\supb{k} + \epsilonb\supb{k}}
              {|\hat \Bcal^k|\lambda_2}\Big)
\end{equation}
\begin{equation}
  \label{eq:sparse_event:T2}
\begin{aligned}
  T_b^3 = \rbr{\tilde \Wb_b}' \bigg[ \Xb\sk\supb{k} &
       \rbr{(\Xb\sk\supb{k})'\Xb\sk\supb{k}}^{-1}
       \Big(\check \yb_{\hat T_{k-1}} + 
         \frac{\lambda_1(\check \zb_{\hat T_{k-1},S^k} 
           - \check \zb_{\hat T_k,S^k})}{|\hat \Bcal^k|\lambda_2}\Big)
       \\&+ \Hb\sk^{\hat \Bcal^k, \perp}\Big(
         \frac{\eb\supb{k} + \epsilonb\supb{k}}
              {|\hat \Bcal^k|\lambda_2}\Big)
      \bigg]
\end{aligned}
\end{equation}
and
\begin{equation}
  \label{eq:sparse_event:T3}
  T_b^4 =  - \frac{\lambda_1(\check z_{\hat T_{k-1},b} 
        - \check z_{\hat T_k,b})}{|\hat \Bcal^k|\lambda_2}.
\end{equation}
We analyze each of the terms separately. Starting with the term
$T_b^1$, after some algebra, we obtain that
\begin{equation}
  \label{eq:tb1:algebra}
\begin{aligned}
  \Fb_b' &\Xb\sk\supb{k}
       \rbr{(\Xb\sk\supb{k})'\Xb\sk\supb{k}}^{-1} \\
  &= 
  \sum_{j\ :\ \hat  \Bcal^k \cap \Bcal^j \neq \emptyset}
  \frac{|\Bcal^j \cap \hat\Bcal^k|}{|\hat \Bcal^k|}
  \Sigmab_{bS^k}^j (\Sigmab_{S^kS^k}^j)^{-1}
    (\hat\Sigmab_{S^kS^k}^{\Bcal^j \cap \hat\Bcal^k} - 
    \Sigmab_{S^kS^k}^j)
    \rbr{{\hat{\tilde\Sigmab}}_{S^kS^k}^{k}}^{-1}\\
  &\quad +
  \tilde\Sigmab_{bS^k}^k(({\hat{\tilde\Sigmab}}_{S^kS^k}^k)^{-1} - 
       ({\tilde\Sigmab_{S^kS^k}}^k)^{-1}) \\
  &\quad +
  \tilde\Sigmab_{bS^k}^k{(\tilde\Sigmab_{S^kS^k}^k)}^{-1}.
\end{aligned}
\end{equation}
Recall that we are working on the event $\Ecal$, so that
$\opnorm{\tilde\Sigmab_{N^kS^k}^k{(\tilde\Sigmab_{S^kS^k}^k)}^{-1}}{\infty}
\xrightarrow{n \rightarrow \infty}
\opnorm{\Sigmab_{N^kS^k}^k{(\Sigmab_{S^kS^k}^k)}^{-1}}{\infty}$ and
$(|\hat \Bcal^k|\lambda_2)^{-1}\lambda_1(\check \zb_{\hat T_{k-1},S^k}
- \check \zb_{\hat T_k,S^k}) \xrightarrow{n \rightarrow \infty} 0$
element-wise. Using \eqref{eq:bound-covariance-elements} we bound the
first two terms in the equation above. We bound the first term by
observing that for any $j$ and any $b \in N^k$ and $n$ sufficiently
large
\begin{equation*}
\begin{aligned}
\frac{|\Bcal^j \cap \hat\Bcal^k|}{|\hat \Bcal^k|}&
\norm{
  \Sigmab_{bS^k}^j (\Sigmab_{S^kS^k}^j)^{-1}
    (\hat\Sigmab_{S^kS^k}^{\Bcal^j \cap \hat\Bcal^k} -
    \Sigmab_{S^kS^k}^j)}_{\infty}  \\ 
&\leq    
\frac{|\Bcal^j \cap \hat\Bcal^k|}{|\hat \Bcal^k|}
\norm{
  \Sigmab_{bS^k}^j (\Sigmab_{S^kS^k}^j)^{-1}}_{1}
\norm{
    \hat\Sigmab_{S^kS^k}^{\Bcal^j \cap \hat\Bcal^k} -
    \Sigmab_{S^kS^k}^j}_{\infty} \\
& \leq  C_1 
\frac{|\Bcal^j \cap \hat\Bcal^k|}{|\hat \Bcal^k|}
\norm{
    \hat\Sigmab_{S^kS^k}^{\Bcal^j \cap \hat\Bcal^k} -
    \Sigmab_{S^kS^k}^j}_{\infty} 
  \leq \epsilon_1
\end{aligned}
\end{equation*}
with probability $1 - c_1 \exp(-c_2\log n)$. Next, for any $b \in N^k$
we bound the second term as
\begin{equation*}
\begin{aligned}
 &\norm{ \tilde\Sigmab_{bS^k}^k(({\hat{\tilde\Sigmab}}_{S^kS^k}^k)^{-1} - 
       ({\tilde\Sigmab_{S^kS^k}}^k)^{-1}) }_1 \\
 &\quad
 \leq C_2 \norm{({\hat{\tilde\Sigmab}}_{S^kS^k}^k)^{-1} - 
       ({\tilde\Sigmab_{S^kS^k}}^k)^{-1})}_F \\
 & \quad
 \leq C_2 
 \norm{\tilde\Sigmab_{S^kS^k}^k}_F^2
    \norm{\hat{\tilde\Sigmab}_{S^kS^k}^k - 
       \tilde\Sigmab_{S^kS^k}^k}_F + 
    \Ocal(\norm{\hat{\tilde\Sigmab}_{S^kS^k}^k - 
       {\tilde\Sigmab_{S^kS^k}}^k}_F^2) \\
 & \leq \epsilon_2
\end{aligned}
\end{equation*}
with probability $1 - c_1 \exp(-c_2 \log n)$. Choosing $\epsilon_1,
\epsilon_2$ sufficiently small and for $n$ large enough, we have that
$\max_b | T_b^1 | \leq 1 - \alpha + o_p(1)$ under the assumption~{\bf
  A4}.

We proceed with the term $T_b^2$, which can be written as
\begin{equation*}
\begin{aligned}
 T_b^2 & = (|\hat \Bcal^k|\lambda_2)^{-1}
  \rbr{
  \Sigmab_{bS^k}^k\rbr{\Sigmab_{S^kS^k}^k}^{-1} - 
  \Fb_b' \Xb\sk\supb{k}
  \rbr{(\Xb\sk\supb{k})'\Xb\sk\supb{k}}^{-1}}
  \sum_{i \in \Bcal^k \cap \hat \Bcal^k} \xb\sk^i \epsilon^i \\
& + (|\hat \Bcal^k|\lambda_2)^{-1}
  \sum_{i \not\in \Bcal^k \cap \hat \Bcal^k}
  \rbr{
  \Sigmab_{bS^k}^{\Bcal(i)}\rbr{\Sigmab_{S^kS^k}^{\Bcal(i)}}^{-1} - 
  \Fb_b' \Xb\sk\supb{k}
  \rbr{(\Xb\sk\supb{k})'\Xb\sk\supb{k}}^{-1}}
   \xb\sk^i (e^i + \epsilon^i). 
\end{aligned}
\end{equation*}
Since we are working on the event $\Ecal$ the second term in the above
equation is dominated by the first term. Next, using
\eqref{eq:tb1:algebra} together with
\eqref{eq:bound-covariance-elements}, we have that for all $b \in N^k$
\begin{equation*}
\norm{
  \Sigmab_{bS^k}^k\rbr{\Sigmab_{S^kS^k}^k}^{-1} - \Fb_b'
  \Xb\sk\supb{k} \rbr{(\Xb\sk\supb{k})'\Xb\sk\supb{k}}^{-1}}_2 = o_p(1).
\end{equation*}
Combining with Lemma~\ref{lem:concentration-error-term}, we have that
under the assumptions of the theorem
\begin{equation*}
  \max_b |T_b^2| = o_p(1).
\end{equation*}

We deal with the term $T_b^3$ by conditioning on $\Xb\sk\supb{k}$ and
$\epsilonb\supb{k}$, we have that $\Wb_b$ is independent of the terms
in the squared bracket in $T_b^3$, since all $\check \zb_{\hat
  T_{k-1},S}, \check \zb_{\hat T_{k},S}$ and $\hat \yb_{\hat
  T_{k-1},S}$ are determined from the solution to the restricted
optimization problem. To bound the second term, we observe that
conditional on $\Xb\sk\supb{k}$ and $\epsilonb\supb{k}$, the variance of
$T_b^3$ can be bounded as
\begin{equation} \label{eq:variance-bound}
\begin{aligned}
  \Var(T_b^3) & \leq \norm{\Xb\sk\supb{k} 
       \rbr{(\Xb\sk\supb{k})'\Xb\sk\supb{k}}^{-1}
       \check \eta\sk
       + \Hb\sk^{\hat \Bcal^k, \perp}\Big(
         \frac{\eb\supb{k} + \epsilonb\supb{k}}
              {|\hat \Bcal^k|\lambda_2}\Big)}_2^2 \\
   & \leq \check \eta\sk' 
       \rbr{(\Xb\sk\supb{k})'\Xb\sk\supb{k}}^{-1}
       \check \eta\sk  + \Big\|
         \frac{\eb\supb{k} + \epsilonb\supb{k}}
              {|\hat \Bcal^k|\lambda_2}\Big\|_2^2,
\end{aligned}
\end{equation}
where
\begin{equation*}
  \check \eta\sk = \Big(\check \yb_{\hat T_{k-1},S^k} + 
  \frac{\lambda_1(\check \zb_{\hat T_{k-1},S^k} 
    - \check \zb_{\hat T_k,S})}{|\hat \Bcal|\lambda_2}\Big).
\end{equation*}
Using lemma~\ref{lem:concentration-eigenvalue:supremum-over-blocks}
and Young's inequality, the first term in \eqref{eq:variance-bound} is
upper bounded by
\begin{equation*}
\frac{18}{|\hat \Bcal|\phi_{\min}}
\rbr{ s + \frac{2\lambda_1^2}{|\hat \Bcal|^2\lambda_2^2}}  
\end{equation*}
with probability at least $1 - 2\exp(-|\hat \Bcal^k|/2 + 2\log
n)$. Using lemma~\ref{lem:chi-square:bound-over-all-partitions} we
have that the second term is upper bounded by 
\begin{equation*}
  \frac{(1+\delta')(1+M^2\phi_{\max})}{|\hat \Bcal|\lambda_2^2}
\end{equation*}
with probability at least $1 - \exp(-c_1|\hat \Bcal^k|\delta'^{2} +
2\log n)$. Combining the two bounds, we have that $\Var(T_b^3) \leq
c_1 s(|\hat \Bcal^k|)^{-1}$ with high probability, using the fact that
$(|\hat \Bcal^k|\lambda_2)^{-1}\lambda_1 \rightarrow 0$ and $|\hat
\Bcal^k| \lambda_2 \rightarrow \infty$ as $n \rightarrow
\infty$. Using the bound on the variance of the term $T_b^3$ and the
Gaussian tail bound, we have that 
\begin{equation*} 
  \max_{b \in N} |T_b^3| = o_p(1).
\end{equation*}

Combining the results, we have that $\max_{b \in N^k} |Y_b| \leq 1 -
\alpha + o_p(1)$. For a sufficiently large $n$, under the conditions
of the theorem, we have shown that $\max_{b \in N}|Y_b| < 1$ which
implies that $\PP[S(\hat \thetab^k) \subset S^k] \xrightarrow{n
  \rightarrow \infty} 1$.

Next, we proceed to show that $\PP[S^k \subset S(\hat \thetab^k)]
\xrightarrow{n \rightarrow \infty} 1$. Observe that
\begin{equation*}
  \PP[S^k \not \subset S(\hat \thetab^k)] \leq 
    \PP[\norm{\hat \thetab_{S^k}^k - \thetab_{S^k}^k}_{\infty}
      \geq \theta_{\min}].
\end{equation*}
From \eqref{eq:kkt-known-boundaries} we have that $\norm{\hat
  \thetab_{S^k}^k - \thetab_{S^k}^k}_{\infty}$ is upper bounded by 
\begin{equation*}
\begin{aligned}
  &\bignorm{ 
    \rbr{\frac{1}{|\hat \Bcal^k|}(\Xb_{S^k}\supb{k})'\Xb_{S^k}\supb{k}}^{-1} 
    \frac{1}{|\hat \Bcal^k|}
    (\Xb_{S^k}\supb{k})'(\tilde \eb\supb{k} + \epsilonb\supb{k})}_\infty \\
  & \qquad \qquad + \bignorm{
     \rbr{(\Xb_{S^k}\supb{k})'\Xb_{S^k}\supb{k}}^{-1} \rbr{
    \lambda_1(\check \zb_{{\hat T_{k-1}},S^k} - \check \zb_{{\hat T_{k}},S^k}) -
    \lambda_2 |\hat \Bcal\supb{k}| \check \yb_{{\hat T_{k-1}},S^k}}}_\infty.
\end{aligned}
\end{equation*}
Since $\tilde e_i \neq 0$ only on $i \in \hat \Bcal^k\bks \Bcal^k$ and
$n\delta_n / |\hat \Bcal^k| \rightarrow 0$, the term involving $\tilde
\eb\supb{k}$ is stochastically dominated by the term involving
$\epsilonb\supb{k}$ and can be ignored. Define the following terms
\begin{equation*}
\begin{aligned}
  T_1 &= \rbr{\frac{1}{|\hat \Bcal^k|}(\Xb_{S^k}\supb{k})'\Xb_{S^k}\supb{k}}^{-1} 
    \frac{1}{|\hat \Bcal^k|}
    (\Xb_{S^k}\supb{k})'\epsilonb\supb{k},\\    
  T_2 &= \rbr{\frac{1}{|\hat \Bcal^k|}(\Xb_{S^k}\supb{k})'\Xb_{S^k}\supb{k}}^{-1} 
    \frac{\lambda_1}{|\hat \Bcal^k|\lambda_2}
    (\check \zb_{{\hat T_{k-1}},S^k} - \check \zb_{{\hat T_{k}},S^k}),\\
  T_3 &= \rbr{\frac{1}{|\hat \Bcal^k|}(\Xb_{S^k}\supb{k})'\Xb_{S^k}\supb{k}}^{-1} 
     \check \yb_{{\hat T_{k-1}},S^k}.
\end{aligned}
\end{equation*}
Conditioning on $\Xb_{S^k}\supb{k}$, the term $T_1$ is a $|S^k|$
dimensional Gaussian with variance bounded by $c_1/n$
with probability at least $1 - c_1\exp(-c_2\log n)$ using
lemma~\ref{lem:concentration-eigenvalue:supremum-over-blocks}. Combining
with the Gaussian tail bound, the term $\norm{T_1}_\infty$ can be
upper bounded as
\begin{equation}
  \PP\bigg[\norm{T_1}_\infty \geq c_1 \sqrt{\frac{\log s}{n}}\bigg] \leq c_2
  \exp(-c_3 \log n).
\end{equation}
Using lemma~\ref{lem:concentration-eigenvalue:supremum-over-blocks},
we have that with probability greater than $1 - c_1\exp(-c_2 \log n)$ 
\begin{equation*}
  \norm{T_2}_\infty \leq \norm{T_2}_2 \leq  c_3 \frac{\lambda_1}{|\hat
    \Bcal^k|\lambda_2} \rightarrow 0
\end{equation*}
under the conditions of theorem. Similarly $\norm{T_3}_\infty \leq
c_1\sqrt{s}$, with probability greater than $1 - c_1\exp(-c_2 \log
n)$. Combining the terms, we have that 
\begin{equation*}
  \norm{\thetab^k - \hat
    \thetab^k}_\infty \leq c_1\sqrt{\frac{\log s}{n}} + c_2 \sqrt{s}\lambda_2
\end{equation*}
with probability at least $1 - c_3\exp(-c_4 \log n)$. Since
$\theta_{\min} = \Omega(\sqrt{\log (n)/ n})$, we have shown that $S^k
\subseteq S(\hat \thetab^k)$. Combining with the first part, it
follows that $S(\hat \thetab^k) = S^k$ with probability tending to one.

\section*{Acknowledgments} 

We are thankful to Za\"{i}d Harchaoui for many useful
discussions. Furthermore, we thank Larry Wasserman and Ankur P. Parikh
for providing comments on an early version of this work and many
insightful suggestions.

\section*{Appendix}
\label{sec:appendix}

\subsection*{Technical results}

In this section we collect some technical results needed for the
proves presented in $\S$\ref{sec:proofs}.

\begin{lemma} \label{lem:chi-square:bound-over-all-partitions} Let
  $\{\zeta^i\}_{i \in [n]}$ be a sequence of iid $\Ncal(0,1)$ random
  variables. If $v_n \geq C \log n$, for some constant $C > 16$, then
  \begin{equation*}
    \label{eq:chi-square:bound-over-all-partitions}
    \PP\bigg[ \bigcap_{{1 \leq l < r \leq n} \atop {r - l > r_n}}
    \Big\{
    \sum_{i=l}^r (\zeta^i)^2 \leq (1 + C)(r-l+1)
    \Big\}\bigg] 
     \geq 1 - \exp(-c_1\log n)
  \end{equation*}
  for some constant $c_1 > 0$. 
\end{lemma}
\begin{proof}
  For any $1 \leq l < r \leq n$, with $r - l > v_n$ we have 
  \begin{equation*}
    \begin{aligned}
      \PP[\sum_{i=l}^r (\zeta^i)^2 \geq (1 + C)(r-l+1)]
      & \leq \exp(-C(r-l+1)/8) \\
      & \leq \exp(-C\log n / 8 )
    \end{aligned}
  \end{equation*}
  using \eqref{eq:chi2_bound}. The lemma follows from an application
  of the union bound.
\end{proof}

\begin{lemma} \label{lem:concentration-error-term} Let $\{\xb_i\}_{i
    \in [n]}$ be independent observations from \eqref{eq:model} and
  let $\{\epsilon_i\}_{i \in [n]}$ be independent $\Ncal(0,
  1)$. Assume that {\bf A1} holds. If $v_n \geq C \log n$ for some
  constant $C > 16$, then
  \begin{equation*}
  \begin{aligned}
    \PP\bigg[\bigcap_{j \in [B]}&
      \bigcap_{{l, r \in \Bcal^j} \atop {r - l > v_n}}\ 
      \bigg \{\frac{1}{r-l+1} \norm{\sum_{i=l}^r \xb_{i} \epsilon_i}_2
      \leq \frac{\phi_{\max}^{1/2}\sqrt{1+C}}{\sqrt{r-l+1}} 
      \sqrt{p(1 + C\log n)} \Big\}\bigg]  \\
    & \geq 1 - c_1\exp(-c_2 \log n),
  \end{aligned}
  \end{equation*}
  for some constants $c_1, c_2 > 0$.
\end{lemma}
\begin{proof}
  Let $\Sigmab^{1/2}$ denote the symmetric square root of the
  covariance matrix $\Sigmab_{SS}$ and let $\Bcal(i)$ denote the block
  $\Bcal^j$ of the true partition such that $i \in \Bcal^j$. With this
  notation, we can write $\xb_{i} = \rbr{ \Sigmab^{\Bcal(i)}
  }^{1/2} \ub_{i}$ where $\ub_{i} \sim \Ncal(\0b, \Ib)$. For any $l \leq
  r \in \Bcal^j$ we have
  \begin{equation*}
    \norm{\sum_{i = l}^r \xb_{i}\epsilon_i}_2 = 
    \norm{\sum_{i = l}^r \rbr{ \Sigmab^{j} }^{1/2}
      \ub_{i}\epsilon_i}_2 \leq     
    \phi_{\max}^{1/2}\norm{\sum_{i = l}^r \ub_{i}\epsilon_i}_2.
  \end{equation*}
  Conditioning on $\{\epsilon_i\}_i$, for each $b \in [p]$, $\sum_{i =
    l}^r u_{i,b}\epsilon_i$ is a normal random variable with variance
  $\sum_{i = l}^r (\epsilon_i)^2$. Hence, $\norm{\sum_{i=l}^r \ub_{i}
    \epsilon_i}_2^2 / (\sum_{i = l}^r (\epsilon_i)^2)$ conditioned on
  $\{\epsilon_i\}_i$ is distributed according to $\chi^2_p$ and 
  \begin{equation*}
  \begin{aligned}
    \PP\bigg[\frac{1}{r-l+1}&\norm{\sum_{i=l}^r \xb_{i}\epsilon_i}_2 \geq 
    \frac{\phi_{\max}^{1/2}\sqrt{\sum_{i=l}^r (\epsilon_i)^2 }}{r-l+1}
    \sqrt{p(1 + C\log n)}\ \Big|\ \{\epsilon_i\}_{i = l}^r \bigg] \\
    & \leq \PP[\chi^2_p \geq p(1 + C\log n)] \leq \exp(-C\log n/8),
  \end{aligned}
  \end{equation*}
  where the last inequality follows from~\eqref{eq:chi2_bound}. Using
  lemma~\ref{lem:chi-square:bound-over-all-partitions}, for all $l, r
  \in \Bcal^j$ with $r - l > v_n$ the quantity $\sum_{i = l}^r
  (\epsilon_i)^2$ is bounded by $(1 + C)(r - l + 1)$ with
  probability at least $1 - \exp(-c_1\log n)$, which
  gives us the following bound
  \begin{equation*}
  \begin{aligned}
    \PP\bigg[\bigcap_{j \in [B]}&
      \bigcap_{{l, r \in \Bcal^j} \atop {r - l > v_n}}\ 
      \bigg \{\frac{1}{r-l+1} \norm{\sum_{i=l}^r \xb_{i} \epsilon_i}_2
      \leq \frac{\phi_{\max}^{1/2}\sqrt{1 + C}}{\sqrt{r-l+1}} 
      \sqrt{p(1 + C\log n)} \Big\}\bigg]  \\
    & \geq 1 - c_1\exp(-c_2 \log n).
  \end{aligned}
  \end{equation*}
\end{proof}
\begin{lemma} \label{lem:concentration-eigenvalue:supremum-over-blocks}
  Let $\{\xb_i\}_{i \in [n]}$ be independent observations from
  \eqref{eq:model}. Assume that {\bf A1} holds. Then for any $v_n >
  p$,
  \begin{equation*}
    \label{eq:concentration-eigenvalue:supremum-over-blocks}
    \PP\Big[\max_{{1 \leq l < r \leq n} \atop {r - l > v_n}}\
    \Lambda_{\max}\rbr{
      \frac{1}{r-l+1} \sum_{i=l}^r \xb_{i}\rbr{\xb_{i}}'}
       \geq 9\phi_{\max} \Big] 
    \leq 2n^2 \exp(-v_n/2)
  \end{equation*}
  and
  \begin{equation*}
    \label{eq:concentration-eigenvalue:infimum-over-blocks}
    \PP\Big[\min_{{1 \leq l < r \leq n} \atop {r - l > v_n}}\
    \Lambda_{\min}\rbr{
      \frac{1}{r-l+1} \sum_{i=l}^r  \xb_{i}\rbr{\xb_{i}}'}
       \leq \phi_{\min}/9 \Big] 
    \leq 2n^2 \exp(-v_n/2).
  \end{equation*}
\end{lemma}
\begin{proof}
  For any $1 \leq l < r \leq n$, with $r - l \geq v_n$ we have 
  \begin{equation*}
  \begin{aligned}
    \PP\Big[\Lambda_{\max} \rbr{
     \frac{1}{r-l+1} \sum_{i=l}^r  \xb_{i}\rbr{\xb_{i}}'}
      \geq 9\phi_{\max} \Big] 
    &\leq 2 \exp(-(r-l+1)/2)\\& \leq 2 \exp(-v_n/2)
  \end{aligned}
  \end{equation*}
  using \eqref{eq:max_eigen_value},
  convexity of $\Lambda_{\max}(\cdot)$ and {\bf A1}. The
  lemma follows from an application of the union bound. The other
  inequality follows using a similar argument.
\end{proof}

\subsection*{Proof of Proposition~\ref{prop:oversegmented-boundaries}}

The following proof follows main ideas already given in
theorem~\ref{thm:consistent-estimation-boundaries:known-num-blocks}. We
provide only a sketch. 

Given an upper bound on the number of partitions $B_{\max}$, we are
going to perform the analysis on the event $\{ \hat B \leq B_{\max}
\}$. Since
\begin{equation*}
  \PP[h(\hat \Tcal, \Tcal) \geq n\delta_n
  \ \big|\ \{ \hat B \leq B_{\max} \}] \leq 
  \sum_{B' = B}^{B_{\max}} \PP[h(\hat \Tcal, \Tcal) \geq n\delta_n\
  \big|\ \{|\hat \Tcal| = B'+1\}],
\end{equation*}
we are going to focus on $\PP[h(\hat \Tcal, \Tcal) \geq n\delta_n\
\big|\ \{|\hat \Tcal| = B'+1\}]$ for $B' > B$ (for $B' = B$ it follows
from
theorem~\ref{thm:consistent-estimation-boundaries:known-num-blocks}
that $h(\hat \Tcal, \Tcal) < n\delta_n$ with high probability). Let us
define the following events 
\begin{align*}
\Ecal_{j, 1} & = \{\exists l \in [B']\ :\ 
|\hat T_l - T_j| \geq n\delta_n, |\hat T_{l+1} -
T_j| \geq n\delta_n \text{ and } \hat T_l < T_j < \hat T_{l+1}\}  \\
\Ecal_{j, 2} & = \{\forall l \in [B']\ :\ |\hat T_l - T_j| \geq
n\delta_n \text{ and } \hat T_l< T_j\} \\
\Ecal_{j, 3} & = \{\forall l \in [B']\ :\ |\hat T_l - T_j| \geq
n\delta_n \text{ and } \hat T_l > T_j\}.
\end{align*}
Using the above events, we have the following bound
\begin{equation*}
  \PP[h(\hat \Tcal, \Tcal) \geq n\delta_n
    \ \big|\ \{|\hat \Tcal| = B'+1\}] \leq 
  \sum_{j \in [B]} \PP[\Ecal_{j,1}] + \PP[\Ecal_{j,2}] +
  \PP[\Ecal_{j,3}].
\end{equation*}
The probabilities of the above events can be bounded using the same
reasoning as in the proof of
theorem~\ref{thm:consistent-estimation-boundaries:known-num-blocks},
by repeatedly using the KKT conditions given in
\eqref{eq:kkt_conditions}. In particular, we can use the strategy used
to bound the event $A_{n,j,2}$. Since the proof is technical and does
not reveal any new insight, we omit the details.

\subsection*{A collection of known results}

This section collects some known results that we have used in the
paper. We start by collecting some results on the eigenvalues of
random matrices. Let $\xb \iidsim \Ncal(0, \Sigmab)$, $i \in [n]$, and
$\hat \Sigmab = n^{-1} \sum \xb_i(\xb_i)'$ be the empirical covariance
matrix. Denote the elements of the covariance matrix $\Sigmab$ as
$[\sigma_{ab}]$ and of the empirical covariance matrix $\hat \Sigmab$
as $[\hat \sigma_{ab}]$.

Using standard results on concentration of spectral norms and
eigenvalues \citep{davidson01local}, \cite{wainwright06sharp} derives
the following two crude bounds that can be very useful.  Under the
assumption that $p < n$,
\begin{align}
  \label{eq:max_eigen_value}
  \PP[\Lambda_{\max}(\hat \Sigmab) \geq 9\phi_{\max}] &\leq 2 \exp(-n/2) \\ 
  \label{eq:min_eigen_value}
  \PP[\Lambda_{\min}(\hat \Sigmab) \leq \phi_{\min}/9] &\leq 2 \exp(-n/2).
\end{align}

From Lemma A.3. in \cite{bickel08regularized} we have the following
bound on the elements of the covariance matrix
\begin{equation}
  \label{eq:bound-covariance-elements}
  \PP[|\hat\sigma_{ab} - \sigma_{ab}| \geq \epsilon] \leq c_1
  \exp(-c_2 n \epsilon^2),\qquad |\epsilon| \leq \epsilon_0
\end{equation}
where $c_1$ and $c_2$ are positive constants that depend only on
$\Lambda_{\max}(\Sigmab)$ and $\epsilon_0$.

Next, we use the following tail bound for $\chi^2$ distribution from
\cite{lounici09taking}, which holds for all $\epsilon > 0$, 
\begin{equation} \label{eq:chi2_bound}
  \PP[\chi^2_n > n + \epsilon] \leq 
   \exp(-\frac{1}{8}\min(\epsilon, \frac{\epsilon^2}{n})).
\end{equation}

\bibliography{biblio}

\end{document}